%% file: main.tex
\documentclass{article}
\usepackage[utf8]{inputenc}
\usepackage[a4paper, total={6in, 8in}]{geometry}
\usepackage{amsthm}
\usepackage{graphicx} 
\usepackage{array} 

\usepackage{amsmath, amssymb, amsfonts, verbatim}
\usepackage{hyphenat,epsfig,subfigure,multirow}
\usepackage{nicefrac}
\usepackage{caption}
\usepackage{array}
\newcolumntype{P}[1]{>{\centering\arraybackslash}p{#1}}

\usepackage{enumitem} 
\setlist[enumerate]{wide = 0pt, leftmargin=*}

\usepackage[usenames,dvipsnames]{xcolor}

\usepackage[bottom]{footmisc}
\usepackage{tcolorbox}
\tcbuselibrary{skins,breakable}
\tcbset{enhanced jigsaw}

\usepackage[normalem]{ulem}
\usepackage[compact]{titlesec}

\definecolor{DarkRed}{rgb}{0.5,0.1,0.1}
\definecolor{RURed}{rgb}{0.8,0.1,0.1}
\definecolor{DarkBlue}{rgb}{0.1,0.1,0.5}

\usepackage{nameref}
\definecolor{ForestGreen}{rgb}{0.1333,0.5451,0.1333}
\definecolor{Red}{rgb}{0.9,0,0}
\usepackage[linktocpage=true,
	pagebackref=true,colorlinks,
	linkcolor=DarkRed,citecolor=ForestGreen,
	bookmarks,bookmarksopen,bookmarksnumbered]
	{hyperref}
\usepackage[noabbrev,nameinlink]{cleveref}
\crefname{property}{property}{Property}
\creflabelformat{property}{(#1)#2#3}
\crefname{equation}{eq}{Eq}
\creflabelformat{equation}{(#1)#2#3}

\usepackage{bm}
\usepackage{url}
\usepackage{xspace}
\usepackage[mathscr]{euscript}
\usepackage{placeins}

\usepackage{tikz}
\usetikzlibrary{arrows}
\usetikzlibrary{arrows.meta}
\usetikzlibrary{shapes}
\usetikzlibrary{backgrounds}
\usetikzlibrary{positioning}
\usetikzlibrary{decorations.markings}
\usetikzlibrary{patterns}
\usetikzlibrary{calc}
\usetikzlibrary{fit}

\newtheorem{theorem}{Theorem}
\newtheorem{lemma}{Lemma}[section]
\newtheorem{proposition}[lemma]{Proposition}

\newtheorem*{claim*}{Claim}
\newtheorem*{proposition*}{Proposition}
\newtheorem*{lemma*}{Lemma}
\newtheorem*{corollary*}{Corollary}
\newtheorem*{remark*}{Remark}

\theoremstyle{definition}

\newtheorem*{problem*}{Problem}
\newtheorem{remark}{Remark}

\newtheorem{mdalg}{Algorithm}

\input{mathshortcut}

\title{Tight Regret Bounds for Single-pass Streaming Multi-armed Bandits}
\author{Chen Wang\footnote{(\href{mailto:\{chen.wang.cs\}@rutgers.edu}{chen.wang.cs@rutgers.edu)} Department of Computer Science, Rutgers University. Research supported in part by NSF CAREER Grant CCF-2047061, a gift from Google Research, and a Rutgers Research Fulcrum Award.}}
\date{}

\begin{document}

\maketitle
\input{abstract}

\input{intro}

\input{prelim}

\input{lower-bound}

\input{high-prob-regret}

\input{expect-regret}

\input{simulation}

\input{conclusion}

\section*{Acknowledgement}
The author would like to thank Michael Saks and Sepehr Assadi of Rutgers University for very helpful discussions and anonymous ICML 2023 reviewers for their constructive comments. 

\bibliographystyle{alpha}
\bibliography{reference}

\appendix
\input{app-tech-prelim}

\input{app-single-arm-eps-best-alg}

\input{app-algorithm}

\end{document}

%% file: mathshortcut.tex
\newcommand{\eps}{\ensuremath{\varepsilon}}

\newcommand{\bracket}[1]{\left[#1\right]}
\newcommand{\paren}[1]{\ensuremath{\left(#1\right)}\xspace}
\newcommand{\card}[1]{\left\vert{#1}\right\vert}

\newcommand{\ceil}[1]{{\left\lceil{#1}\right\rceil}}

\newcommand{\expect}[1]{\Exp\bracket{#1}}

\newcommand{\set}[1]{\ensuremath{\left\{ #1 \right\}}}
\newcommand{\poly}{\mbox{\rm poly}}

\newcommand{\ALG}{\ensuremath{\textnormal{\sf ALG}}\xspace}
\newcommand{\alg}{\ensuremath{\mathcal{A}}\xspace}

\DeclareMathOperator*{\Exp}{\ensuremath{{\mathbb{E}}}}
\DeclareMathOperator*{\Prob}{\ensuremath{\textnormal{Pr}}}
\renewcommand{\Pr}{\Prob}

\newenvironment{tbox}{\begin{tcolorbox}[
		enlarge top by=5pt,
		enlarge bottom by=5pt,
		 breakable,
		 boxsep=0pt,
                  left=4pt,
                  right=4pt,
                  top=10pt,
                  arc=0pt,
                  boxrule=1pt,toprule=1pt,
                  colback=white
                  ]
	}
{\end{tcolorbox}}

\newcommand{\event}{\ensuremath{\mathcal{E}}}

\newcommand{\istar}{\ensuremath{i^{\star}}}

\newcommand{\muhat}{\widehat{\mu}}

\newcommand{\algarm}{\ensuremath{\textnormal{\textsc{Game-of-Arms}}}\xspace}

\newcommand{\hardstoredist}{\ensuremath{\texttt{DIST}}}

\newcommand{\arm}{\ensuremath{\textnormal{\textsf{arm}}\xspace}}
\newcommand{\armstar}{\ensuremath{\arm^{*}}\xspace}

\newcommand{\mustar}{\ensuremath{\mu^{*}}\xspace}

\newcommand{\logstar}{\ensuremath{\log^{*}}\xspace}

\newcommand{\evefindeps}[1]{\ensuremath{\mathcal{E}_{#1,\eps}}\xspace}
\newcommand{\eveepsell}{\ensuremath{\mathcal{A}_{\ell}}\xspace}

\newcommand{\Rtexp}{\ensuremath{R_{T}(\text{explore})}\xspace}
\newcommand{\Rtcom}{\ensuremath{R_{T}(\text{commit})}\xspace}

\newcommand{\instance}{\ensuremath{\textsf{INST}}\xspace}

%% file: abstract.tex
\begin{abstract}
Regret minimization in streaming multi-armed bandits (MABs) has been studied extensively in recent years. In the single-pass setting with $K$ arms and $T$ trials, a regret lower bound of $\Omega(T^{2/3})$ has been proved for any algorithm with $o(K)$ memory (Maiti et al. [NeurIPS'21]; Agarwal at al. [COLT'22]). On the other hand, however, the previous best regret upper bound is still $O\paren{K^{1/3} T^{2/3}\log^{1/3}(T)}$, which is achieved by the streaming implementation of the simple uniform exploration. The $O(K^{1/3}\log^{1/3}(T))$ gap leaves the open question of the tight regret bound in the single-pass MABs with sublinear arm memory.

In this paper, we answer this open problem and complete the picture of regret minimization in single-pass streaming MABs. We first improve the regret lower bound to $\Omega(K^{1/3}T^{2/3})$ for algorithms with $o(K)$ memory, which matches the uniform exploration regret up to a logarithm factor in $T$. We then show that the $\log^{1/3}(T)$ factor is not necessary, and we can achieve $O(K^{1/3}T^{2/3})$ regret by finding an $\eps$-best arm and committing to it in the rest of the trials. For regret minimization with high constant probability, we can apply the single-memory $\eps$-best arm algorithms in Jin et al. [ICML'21] to obtain the optimal bound. Furthermore, for the \emph{expected} regret minimization, we design an algorithm with a single-arm memory that achieves $O(K^{1/3} T^{2/3}\log(K))$ regret, and an algorithm with $O(\logstar(n))$-memory with the optimal $O(K^{1/3} T^{2/3})$ regret following the $\eps$-best arm algorithm in Assadi and Wang [STOC'20].

We further tested the empirical performances of our algorithms on simulated MABs instances. The simulation results show that the proposed algorithms consistently outperform the benchmark uniform exploration algorithm by a large margin, and on occasions reduce the regret by up to 70\% (i.e. 30\% of the regret produced by uniform exploration).
\end{abstract}

%% file: intro.tex
\newcommand{\low}[1]{N_{low}(#1)}

\newcommand{\Isolated}[2]{\ensuremath{{\textnormal{\texttt{Isolated}}_{#2}}}(#1)}
\newcommand{\isolated}[1]{\Isolated{#1}{\eps}}

\newcommand{\Dense}[3]{\ensuremath{{\textnormal{\texttt{Dense}}_{#2,#3}}}(#1)}
\newcommand{\Low}[2]{\ensuremath{\textnormal{\texttt{Low}}_{#2}}(#1)}
\renewcommand{\low}[1]{\ensuremath{\textnormal{\texttt{Low}}_{\eps}}(#1)}

\newcommand{\Kernel}[3]{\ensuremath{\textnormal{\texttt{Kernel}}_{#2,#3}}(#1)}
\newcommand{\kernel}[1]{\Kernel{#1}{\eps}{\delta}}

\newcommand{\Candid}[3]{\CC_{#2,#3}(#1)}

\newcommand{\candid}{\Candid{G}{\eps}{\delta}}

\newcommand{\NS}[1]{\ensuremath{N_{\textnormal{sample}}(#1)}}

\renewcommand{\deg}[1]{\ensuremath{\textnormal{{deg}}(#1)}}
\newcommand{\degp}[1]{\ensuremath{\textnormal{{deg}}^{+}(#1)}}
\newcommand{\degn}[1]{\ensuremath{\textnormal{{deg}}^{-}(#1)}}

\newcommand{\sample}{\ensuremath{\textnormal{\texttt{Sample}}}\xspace}

\newcommand{\Vsparse}{V_{\text{sparse}}}
\newcommand{\Vlight}{V_{\text{light}}}
\newcommand{\Vlowsparse}{V_{\text{low-sparse}}}

\newcommand{\Esparse}[1]{E^{+}_{#1\text{-sparse}}}
\newcommand{\Esparsetild}[1]{\tilde{E}^{+}_{#1\text{-sparse}}}
\newcommand{\Edense}[1]{E^{-}_{#1\text{-dense}}}
\newcommand{\Edensev}[2]{E^{-}_{#1\text{-dense}, \, #2}}
\newcommand{\Edensehat}[1]{\widehat{E}^{-}_{#1\text{-dense}}}
\newcommand{\Edensetild}[1]{\tilde{E}^{-}_{#1\text{-dense}}}

\newcommand{\tC}[1]{\widetilde{C}(#1)}

\newcommand{\CC}{\mathcal{C}}

\newcommand{\sym}{\,\triangle\,}

\newcommand{\edgecost}[1]{\ensuremath{\textnormal{\textsf{edge-cost}}(\, #1)}\xspace}

\section{Introduction}
\label{sec:intro}
The stochastic multi-armed bandits (MABs) is a classical model in machine learning and theoretical computer science that captures various real-world applications. The model was first introduced by Robbins~\cite{Robbins1952} for more than 70 years ago; since then, extensive research efforts have been devoted to two main problems under this model: pure exploration and regret minimization. Both problems start with a collection of $K$ arms with unknown sub-Gaussian reward distributions. In pure exploration, we are interested in finding the best arm, defined as the arm with the highest mean reward, with as small as possible number of arm pulls \cite{EvenDarMM02,MannorT03,AudibertBM10,KarninKS13,JamiesonMNB14,chen2015optimal,KaufmannCG16,AgarwalAAK17,ChenLQ17}. On the other hand, in regret minimization, we are additionally given a parameter $T$ as the total number of trials -- also known as the `horizon' -- and we are interested in generating a plan for $T$ arm pulls to minimize the cumulative reward gap compared to the perfect plan that puts all $T$ pulls on the best arm \cite{thompson1933likelihood,berry1985bandit,bubeck2012regret,komiyama2015optimal,LiauSPY18,Slivkins19,DongMR19,ChaudhuriK20,MaitiPK21,AgarwalKP22}. Although the two lines of research are developed relatively independently, they both have found rich applications like experiment design~\cite{Robbins1952,chow2008adaptive}, search ranking~\cite{AgarwalCEMPRRZ08,radlinski2008learning}, economics~\cite{SaureZ13,KremerMP13}, to name a few.

In recent years, with the strong demand to process massive data, the study of multi-armed bandits under the \emph{streaming} model has attracted considerable attention \cite{LiauSPY18,ChaudhuriK20,assadi2020exploration,JinH0X21,MaitiPK21,AgarwalKP22}. Under this model, the arms arrive one after another in a stream, and the algorithm is only allowed to store a number of arms substantially smaller than $K$. In the single-pass streaming setting, if an arm is not stored or discarded from the memory, it cannot be retrieved later and is lost forever. We shall assume the order of the stream is generated by an adversary, i.e. the worst-case order. The model is a natural adaptation of the MABs to the streaming problems studied extensively in algorithms (e.g.~\cite{AlonMS96,HenzingerRR98,GuhaMMO00,McGregor14}). 

The limited memory poses unique challenges for algorithms under this setting. Indeed, for the regret minimization application, under the classical (RAM) setting, a worst-case regret of $\Theta(\sqrt{KT})$ is necessary and achievable. However, in the single-pass streaming setting, if an algorithm is only given $o(K)$ arm memory, the recent work of \cite{MaitiPK21,AgarwalKP22} proved that an $\Omega(T^{2/3})$\footnote{\cite{MaitiPK21} includes another regret lower bound of $\Omega(K^{1/3}T^{2/3}/m^{7/3})$, where $m$ is the memory of the streaming algorithm. However, their bound is only almost-tight when $m=n^{o(1)}$.} regret is inevitable. Since $T$ could be (and is usually) much larger than $K$, these results already separated the regret minimization under the classical vs. the streaming settings.

Despite the progress on the lower bounds, to the best of our knowledge, there is only limited exploration on the algorithms for single-pass regret minimization. \cite{MaitiPK21} noted that a streaming implementation of the folklore uniform exploration algorithm, which pulls each arm $O((T/K)^{2/3}\log^{1/3}(T))$ times and commits to the arm with the best average empirical reward, achieves $O(K^{1/3} T^{2/3}\log^{1/3}(T))$ expected regret. Moreover, \cite{AgarwalKP22} proposed a (multi-pass) algorithm with $O(T^{2/3} \sqrt{K\log(T)})$ regret in a single pass, but it is clearly sub-optimal in the single-pass setting. As such, there remains an $O((K\log(T))^{1/3})$ gap between the upper and lower bounds.

\paragraph{Our Contributions.} We close this gap and complete and picture for regret minimization in single-pass MABs in this work. In particular, we first tighten the regret lower bound for any algorithm with $o(K)$ memory to $\Omega(K^{1/3}T^{2/3})$ -- this effectively reduces the gap between the upper and lower bounds to $\log^{1/3}(T)$. We then find that this logarithmic factor is not essential: by using an $\eps$-best arm algorithm with $O(K/\eps^2)$ arm pulls, setting $\eps=(K/T)^{1/3}$, and committing to the returned arm for the rest of the trials, we can already get a regret of $O(K^{1/3}T^{2/3})$ with high constant probability. 

To explore algorithms that achieve \emph{expected} optimal regret of $O(K^{1/3}T^{2/3})$, we investigate the case when the $\eps$-best arm algorithm \emph{fail}. To elaborate further, the $\eps$-arm algorithms usually succeed with \emph{constant} probability, and by the tightness of concentration bounds, it is necessary to pay an extra $O(\log(K))$ factor if we want $1-\poly(1/K)$ success probability. However, doing so will inevitably introduce an $O(\log(K))$ multiplicative factor on the regret. As such, we proceed differently by observing a smooth failure probability property for a large family of $\eps$-best arm algorithms. On the high level, for many algorithms, even if it does not return an $\eps$-best arm, it could still capture a $2\eps$-best arm with a high probability, as opposed to return an absolutely low-reward arm. We use this observation to prove a smooth-failure bounded-regret lemma, and use it to devise an algorithm with $O(K^{1/3}T^{2/3})$ expected regrets and a memory of $O(\logstar(K))$ arms.

Our results imply that the ``right way'' to minimize regret in single-pass streaming is to use optimal pure exploration algorithms. This establishes a connection between pure exploration and regret minimization tasks. Previous work like \cite{DegenneNCP19} studied such a connection for the MABs in the offline setting; however, to the best of our knowledge, our results are the first to find the connection in the streaming setting, and it can be of independent interests. 

\paragraph{Experiments.} We evaluate the performances of the $\eps$-best arm-based algorithm with simulated Bernoulli arms. We find that under various initialization of arm instances, the $\eps$-best arm-based algorithms can consistently produce smaller regret comparing to the benchmark. In particular, we find the most stable and competitive algorithm can produce up to 70\% of regret reduction, and the average regret, even account of the outliers, is at most 70\% of the benchmark regret (i.e. 30\% reduction). The codes of the experiment are available on \href{https://github.com/jhwjhw0123/streaming-regret-minimization-MABs}{https://github.com/jhwjhw0123/streaming-regret-minimization-MABs}.

\paragraph{Additional discussions about the streaming MABs model.} The original motivation for \cite{assadi2020exploration} to introduce the model was to capture the large-scale applications of MABs. For example, in the search ranking of online retailers, each arm can be viewed as a product that arrives every hour or minute. For memory efficiency, we only want to store a limited number of products to find the best seller. We further remark that some other problems inherently require storing few arms, even when memory is not a major concern. For example, in crowdsourcing, each arm can be viewed as a solution or a model, and storing all of them may cause management issues.


\subsection{Related Work}
We focus on regret minimization in streaming MABs in this work; nonetheless, it is worth mentioning that the pure exploration problem in the streaming setting also enjoys rich literature. The streaming pure exploration MABs was first introduced and studied by \cite{assadi2020exploration}, and together with the work of \cite{MaitiPK21}, there are known algorithms that finds an $\eps$-best arms with $O(\log(K))$, $O(\log\log(K))$, $O(\logstar(K))$, and $O(1)$ memory and $O(K/\eps^2)$ arm pulls\footnote{Their $O(1)$-memory algorithm includes an additive $\Theta(\log^{2}(K)/\eps^3)$ term.}. \cite{JinH0X21} later introduced an algorithm with a single-arm memory to find an $\eps$-best arms with $O(K/\eps^2)$ pulls, and they also studied algorithms in the multi-pass settings. The single pass pure exploration lower bound was developed recently by \cite{AWneurips22}. We remark that our algorithms and lower bounds are heavily inspired by the techniques developed by the aforementioned work. Furthermore, since we used $\eps$-best arm algorithms as a subroutine in our upper bounds, our work also establish an interesting connection between the pure exploration and the regret minimization objectives.

In addition to the single-pass setting, the regret minimization problem is studied through the lens of multi-pass streams. In fact, earlier algorithms of \cite{LiauSPY18,ChaudhuriK20} all focus on regret minimization under the multi-pass settings (\cite{ChaudhuriK20} additionally requires random-order stream). In light of this, \cite{AgarwalKP22} provides the upper and lower regret bounds that are tight in $T$. In particular, they show that any $P$-pass algorithm with memory $o(K/P^2)$ has to incur $\Omega(T^{2^{P}/(2^{P+1}-1)}/2^P)$ regret; and there exists an algorithm with $O(T^{2^{P}/(2^{P+1}-1)}\sqrt{KP\log(T)})$ regret and $O(1)$-arm memory. Compared to their bounds, our results only apply to the single-pass, but it is tight in all asymptotic terms.

Finally, the MABs algorithms with limited memory is also explored under other models, and there are problems in the streaming setting that are closely related to MABs. For instance, \cite{TaoZZ19,Karpov20FOCS} studies the pure exploration MABs in the distributed settings, which is related to the collaborative learning with limited rounds. Furthermore, a recent line of work \cite{SrinivasWXZ22,PengZ23} studies the streaming expert problem, where the arriving elements are the predictions from the experts. Both their model and ours have applications on online learning, yet we emphasize on different aspects.

%% file: prelim.tex
\subsection{Preliminaries}
\label{sec:prelim}
We introduce the model, the parameters, and the problem we studied in this paper in this section. 

\paragraph{Streaming multi-armed bandits model.} To begin with, we define the streaming MABs model as follows. We consider a collection of $K$ arms with unknown sub-Gaussian reward distributions, and they arrive one after another in a stream. The algorithm can pull an arriving arm arbitrarily many times and decide whether to store it. Furthermore, the algorithm can pull a stored past arm at any point and discard some arms to free up memory when necessary. However, in the single-pass setting, an arm that is not stored or discarded is lost forever. For each arm $\arm_{i}$, we let $\mu_{i}$ be the mean of its reward distribution. We say $\mustar=\max_{i \in [K]} \mu_{i}$ is the \emph{optimal reward} and the arm whose reward is $\mustar$ is the \emph{best arm}, denoted as $\armstar$. 

\paragraph{Regret minimization.} The regret minimization problem in stochastic multi-armed bandits goes as follows:  For the regret minimization problem, we are given a fixed number of trials $T$ (known as the \emph{horizon}) and we want to spend as many trials as possible on the best arm. In particular, suppose algorithm $\alg$ pulls $\arm_{\alg(t)}$ in the $t$-th exploration, we define the \emph{regret} of this trial as
\begin{align*}
r_{t}:= \mustar- \mu_{\alg(t)}.
\end{align*}
And we define the \emph{total expected regret} as
\begin{align*}
\expect{R_{T}}:=\expect{\sum_{t=1}^{T} \mustar- \mu_{\alg(t)}},
\end{align*}
where the expectation is taken over the randomness of the arm pulls and (possibly) the algorithm. Our objective is to minimize the total expected regret. We can analogously define the minimization of probabilistic regret $R_{T}$ over the randomness of the arm pulls and (possibly) the algorithm.

\paragraph{$\eps$-best arm.} We do \emph{not} study $\eps$-best arm algorithms in this paper, but rather use them as blackbox subroutines for the regret minimization purpose. In particular, an $\eps$-best arm algorithm (also known as a $\textsf{PAC}(\eps,\delta)$ algorithm) aims to return an arm whose reward is close to $\mu^*$. More formally, the guarantee of an $\eps$-best arm algorithm is to output with probability at least $1-\delta$ an arm with reward $\mu_{\eps}$, such that $\mu^*-\mu_{\eps}\leq \eps$.

\paragraph{Assumption of $T\geq K$.}
We assume w.log. in this paper that $T\geq K$, and repeatedly use this property in the proofs. Note that if $T<K$, we can easily get an upper bound of $O(T)$ by pulling an arbitrary arm for $T$ times, and the bound is tight since we can easily construct an instance such that with $\Omega(1)$ probability the best arm is never pulled. As such, the tight regret bound becomes $\Theta(T)$ and it is not interesting in neither theory nor practice.

We defer the technical preliminaries to \Cref{sec:tech-prelim}.

%% file: lower-bound.tex
\section{The Tight Regret Lower Bound}
\label{sec:regret-lb}
We now formally state our lower bound result as follows.
\begin{theorem}
\label{thm:lb-single-pass-regret}
There exists a family of streaming stochastic multi-armed bandit instances, such that for any given parameter $T$ and $K$ such that $T\geq K$, any single-pass streaming algorithm with a memory of $\frac{K}{20}$ arms has to suffer 
\begin{align*}
\expect{R_T} \geq C\cdot K^{\frac{1}{3}}\cdot T^{\frac{2}{3}}
\end{align*}
total expected regret for some constant $C$. Furthermore, the lower bound holds even the order of arrival for the arms is uniformly at random.
\end{theorem}

To prove \Cref{thm:lb-single-pass-regret}, we will use a recent result in \cite{AWneurips22} which captures a sample-space trade-off to `trap' the best arm with limited memory.
\begin{proposition}[\cite{AWneurips22}]
\label{prop:arm-trapping-lemma}
Consider the following distribution of $K'$ arms.
\begin{tbox}
$\hardstoredist(K', \beta)$: \textbf{A hard distribution with $K'$ arms for trapping the best arm}
\begin{enumerate}
\item An index $\istar$ sampled uniform at random from $[K']$.
\item For $i\neq \istar$, let the arms be with reward $\mu_{i}= \frac{1}{2}$.
\item For $i=\istar$, let the arm be with reward $\mu_{\istar}=\frac{1}{2}+\beta$.
\end{enumerate}
\end{tbox}
Then, any algorithm that outputs (the indices of) $\frac{K'}{8}$ arms which contains the best arm on $\hardstoredist$ with probability at least $\frac{2}{3}$ has to use at least $\frac{1}{1200}\cdot \frac{K'}{\beta^{2}}$ arm pulls. 
\end{proposition}

One can refer to \cite{AWneurips22} for the proof of \Cref{prop:arm-trapping-lemma}. We note that a similar sample-space trade-off result was proved and used by \cite{AgarwalKP22} in the multi-pass setting. However, their result does not factor in the dependency on $K$, which creates the gap between the upper and the lower regret bounds.

\subsubsection*{Proof of \Cref{thm:lb-single-pass-regret}}
By Yao's minimax principle, to prove lower bounds for randomized algorithm, it suffices to consider deterministic algorithms over a certain distribution of inputs. As such, in what follows, we only consider lower bounds for deterministic algorithms over input family of instances. Our hard distribution of instances is constructed as follows.
\begin{tbox}
\textbf{A hard distribution for single-pass streaming MABs regret minimization}
\begin{enumerate}
\item For the first $\frac{K}{2}$ arms, sample a set of arms from $\hardstoredist(K/2, \Delta)$, where $\Delta= \frac{1}{8}\cdot (\frac{K}{T})^{1/3}$.
\item For the last $\frac{K}{2}$ arms, set all arms except the last ($K$-th) with reward $\frac{1}{2}$.
\item The last arm follows the distribution
\begin{enumerate}
\item With probability $\frac{1}{2}$, set $\mu_{K}=\frac{1}{2}$;
\item With probability $\frac{1}{2}$, set $\mu_{K}=\frac{3}{4}$.
\end{enumerate}
\end{enumerate}
\end{tbox}
For any algorithm $\alg$ with memory at most $\frac{K}{20}$, we analyze the two cases based on whether the algorithm uses at least $\frac{1}{2400}\cdot \frac{K}{\Delta^2}$ arm pulls on the \emph{first} half of the stream. Note that this is the necessary number of arm pulls for $\alg$ to store the best arm among the first half with probability at least $\frac{2}{3}$, i.e. if the algorithm uses less than the above quantity, it cannot keep the arm with reward $\frac{1}{2}+\Delta$ after the first half of the arms with probability at least $\frac{1}{3}$.

\paragraph{Case A). $\alg$ uses at least $\frac{1}{2400}\cdot \frac{K}{\Delta^2}$ arm pulls on the first $\frac{K}{2}$ arms.} In this case, with probability $\frac{1}{2}$, the last arm is with reward $\frac{3}{4}$. As such, each arm pull spent on the first $\frac{K}{2}$ arms incurs a regret of at least $(\frac{1}{4}-\Delta)$. As such, the expected regret is at least
\begin{align*}
\expect{R_{T}} &\geq \Pr(\mu_{K}=\frac{3}{4})\cdot \expect{R_{T}\mid \mu_{K}=\frac{3}{4}} \\
&\geq \frac{1}{2}\cdot \frac{1}{2400}\cdot \frac{K}{\Delta^2} \cdot (\frac{1}{4}-\Delta)\\
&\geq \frac{1}{2}\cdot \frac{1}{2400}\cdot \frac{K}{\Delta^2} \cdot \frac{1}{8} \tag{$K\leq T$ implies $\Delta \leq \frac{1}{8}$}\\
&= \Omega(1)\cdot K^{1/3} T^{2/3}.
\end{align*}

\paragraph{Case B). $\alg$ uses less than $\frac{1}{2400}\cdot \frac{K}{\Delta^2}$ arm pulls on the first $\frac{K}{2}$ arms.} In this case, with probability $\frac{1}{2}$ $\arm_{K}$ is with reward $\mu_{K}=\frac{1}{2}$; and since the memory of $\alg$ is $K/20<\frac{K/2}{8}$, by \Cref{prop:arm-trapping-lemma}, with probability at least $\frac{1}{3}$, $\alg$ does not keep the arm with reward $\frac{1}{2}+\Delta$ in the memory upon reading the $(\frac{K}{2}+1)$-th arm. As such, we define the event
\begin{center}
$\event$: $\mu_{K}=\frac{1}{2}$ and $\alg$ does \emph{not} keep the arm with reward $\frac{1}{2}+\Delta$ after reading the first $K/2$ arms
\end{center}
and we have $\Pr(\event)\geq \frac{1}{6}$. Conditioning on $\event$, every arm pull after reading the $(\frac{K}{2}+1)$-th arm incurs a regret of $\Delta$, and there are at least $(T-\frac{1}{2400}\cdot \frac{K}{\Delta^2})$ trials left. As such, the expected regret is at least
\begin{align*}
\expect{R_{T}} &\geq \Pr(\event)\cdot \expect{R_{T}\mid \event} \\
&\geq \frac{1}{6}\cdot (T-\frac{1}{2400}\cdot \frac{K}{\Delta^2}) \cdot \Delta \\
&= \frac{1}{6}\cdot (\frac{1}{8}\cdot K^{1/3} T^{2/3} - \frac{8}{2400} \cdot K^{2/3}T^{1/3}) \tag{by the choice of $\Delta$} \\
&\geq \frac{1}{60}\cdot K^{1/3} T^{2/3}. \tag{$K^{1/3}T^{2/3}\geq K^{2/3}T^{1/3}$}
\end{align*}

\paragraph{Wrapping up the proof.} Any deterministic algorithm $\alg$ with a memory at most $\frac{K}{20}$ has to either fall in case A) or B). As such, the total expected regret is at least $C\cdot K^{1/3} T^{2/3}$ for a fixed constant $C$, as claimed (for the adversarial arrival case).

Finally, for the random order of arrival, note that by applying a random permutation to the hard distribution, with probability $\frac{1}{4}$, the arm with $\frac{1}{2}+\Delta$ is among the first $\frac{K}{2}$ arms and the arm with reward $\mu_{K}$ is among the latter $\frac{K}{2}$ arms. As such, by conditioning on such an event, the total expected regret becomes asymptotically the same (smaller by a $\frac{1}{4}$ factor). 

%% file: high-prob-regret.tex
\section{The Tight Probabilistic Regret Upper Bound}
\label{sec:high-prob-ub}

We now turn to the upper bound results. As a first step. we show the easier result for probabilistic regret minimization. Our observation is to attain the $O(K^{1/3}T^{2/3})$ regret, we only need to find an $\eps$-best arm with $\eps=(\frac{K}{T})^{1/3}$. As such, the problem can be solved in a single pass with a single-arm memory.

\begin{theorem}
\label{thm:ub-probabilistic}
There exists a single-pass streaming algorithm that given a stream of stochastic multi-armed bandits and the parameters $T$ and $K$, pulls the arms $T$ times using a single-arm memory, and achieves regret
\begin{align*}
R_T \leq \paren{2\log(1/\delta)+1} \cdot K^{\frac{1}{3}}\cdot T^{\frac{2}{3}}
\end{align*}
with probability at least $1-\delta$ over the randomness of arm pulls.
\end{theorem}

Our algorithm for \Cref{thm:ub-probabilistic} crucially relies on the $\eps$-best arm algorithm in \cite{assadi2020exploration,JinH0X21}. The guarantee of the algorithms can be summarized as follows.
\begin{proposition}[\cite{assadi2020exploration,JinH0X21}]
\label{prop:eps-best-arm-alg}
There exists a single-pass streaming algorithm that given a stream of stochastic multi-armed bandits, an error parameter $\eps\in (0,1)$, and a confidence parameter $\delta \in (0,1)$, with probability at least $1-\delta$ return an arm with reward $\mu_{\eps}$ such that 
$\mu_{\eps} \geq \mu^{*}-\eps$ with $O(\frac{K}{\eps^2}\log(\frac{1}{\delta}))$ arm pulls and a memory of a single arm.
\end{proposition}

A self-contained description of the algorithm in \cite{JinH0X21} can be found in \Cref{app:eps-best-arm-alg}. 
We now show that by picking the appropriate $\eps$, it is straightforward to attain the $O(K^{1/3}T^{2/3})$ regret for any constant probability.

\begin{proof}[Proof of \Cref{thm:ub-probabilistic}]
The algorithm is simply as follows.
\begin{tbox}
\begin{enumerate}
\item Run the algorithm in \Cref{prop:eps-best-arm-alg} with parameter $\eps=\frac{1}{2}\cdot (\frac{K}{T})^{1/3}$, obtain $\arm_{\eps}$.
\item Commit to $\arm_{\eps}$ for all the remaining trials.
\end{enumerate}
\end{tbox}
It is easy to see the algorithm only requires a single-arm memory. As such, we only need to analyze the regret. Note that the regret to find the $\eps$-best arm is at most
\begin{align*}
\frac{2K}{(\frac{K}{T})^{2/3}} \cdot \log(\frac{1}{\delta}) = 2\log(1/\delta) \cdot K^{1/3}T^{2/3}.
\end{align*}
On the other hand, conditioning on the algorithm succeeds, which happens with probability $1-\delta$, the reward gap between the best arm and the arm we commit to is at most $(\frac{K}{T})^{1/3}$. As such, the total regret is at most 
\begin{align*}
T\cdot (\frac{K}{T})^{1/3} = K^{1/3}T^{2/3}.
\end{align*}
Summing up the two regret terms gives us the desired statement.
\end{proof}

\begin{remark}
Note that the upper bound is tight for the probabilistic regret minimization problem -- the lower result in the separate note shows that for any instance in the adversarial family, the regret is at least $\Omega(K^{1/3}T^{2/3})$ with probability $\Omega(1)$. As such, we should not expect any algorithm that is asymptotically better than the guarantee of \Cref{thm:ub-probabilistic}.
\end{remark}


%% file: expect-regret.tex
\section{The Tight Expected Regret Upper Bound}
\label{sec:expect-ub}
The algorithm in \Cref{thm:ub-probabilistic} gives the optimal upper bound for regret minimization in the probabilistic manner. However, one can easily spot that the regret is not optimal in \emph{expectation}. In fact, since the algorithms in \cite{assadi2020exploration,JinH0X21} does not provide any guarantee if the algorithm fails, if the failure probability is a constant ($\delta=\Omega(1)$), the expected regret becomes at least $\delta\cdot T=\Omega(T)$. One can balance the parameters between the success and failure case to achieve an expected regret of $O(K^{1/3}T^{2/3}\log(\frac{T}{K}))$\footnote{Concretely, by setting $\delta=(\frac{K}{T})^{1/3}$, the expected regret is $O\paren{(1-(\frac{K}{T})^{1/3})\cdot \log\paren{(\frac{K}{T})^{1/3}} K^{1/3}T^{2/3} + T\cdot (\frac{K}{T})^{1/3}}$, which is upper-bounded by $O(K^{1/3}T^{2/3}\log(\frac{T}{K}))$.} -- although it is already an improvement over the best-known uniform exploration, the bound is still far from being tight especially when $T >> K$. As such, we need a separate investigation of the optimal algorithm for \emph{expected regret}.

We observe that the only drawback of the exploration-and-committing strategy in \Cref{sec:high-prob-ub} is the failure case since no guarantees is provided by existing algorithms. However, if the algorithm always keep the arm with the best empirical reward, it should not be the case that whenever the algorithm fails, it returns an absolute garbage. As such, the hope here is to obtain \emph{smooth} probabilistic guarantees from existing $\eps$-best arm algorithms to attain the optimal regret bound. 

In what follows, we proceed our main upper bound result by first showing that if the smooth probabilistic guarantee holds, we can indeed obtain algorithms with low regret (\Cref{lem:smooth-failure-regret}). Subsequently, we present two algorithm with expected regret $O(K^{1/3}T^{2/3}\log(K))$ and $O(K^{1/3}T^{2/3})$, respectively. Both bounds utilize $\eps$-best arm algorithms as subroutines -- the first bound employs a variate of the simple naive uniform elimination algorithm, while the second bound uses a more involved algorithm by \cite{assadi2020exploration} and \cite{MaitiPK21}.


\subsection{A Smooth-Failure Bounded-Regret Lemma}
We first present a technical lemma that gives a regret upper bound provided an $\eps$-best arm algorithms that display a `smooth trade-off' between the arm reward and the failure probability. The formal statement of the lemma is as follows.
\begin{lemma}[Smooth-Failure Bounded-Regret Lemma]
\label{lem:smooth-failure-regret}
Let $\instance$ be a streaming multi-armed bandit instance with fixed parameters $T$, $K$ such that $T>K$, and let $\ALG$ be a streaming algorithm that given parameter $\eps$, uses $S$ space and $\frac{M}{\eps^2}$ arm pulls to returns an $\arm_{\ALG(\instance)}$ such that
\begin{align*}
\Pr\paren{\mu_{\ALG(\instance)} < \mu^{*} - c \cdot \eps} \leq (\frac{1}{2})^{c} \cdot \frac{1}{10}.
\end{align*}
for any integer $c\geq 1$. 
Then, there exists an $S$-space streaming algorithm that achieves $O(\frac{M}{K^{2/3}}T^{2/3}+K^{1/3}T^{2/3})$ regret in expectation, i.e.
\begin{align*}
\expect{R_{T}} \leq O(M\cdot \frac{T^{2/3}}{K^{2/3}}+K^{1/3}T^{2/3}).
\end{align*}
\end{lemma}
\begin{proof}
The algorithm is to simply run the streaming algorithm for the $\eps$-best arm with $\eps=(\frac{K}{T})^{1/3}$ (the exploration phase), and commit to the returned arm $\arm_{\ALG(\instance)}$ for the rest of the trials if there is any remaining trials (the committing phase). As such, the space bound trivially follows since we do not use any extra space.

We now analyze the expected regret. To proceed, we let $\Rtexp$ be the regret induced by the exploration phase, and $\Rtcom$ be the regret induced by the committing phase. By the choice of the parameter $\eps$, we (deterministically) have $\Rtexp\leq \frac{M}{\eps^2}= M\cdot \frac{T^{2/3}}{K^{2/3}}$, which implies
\[\expect{\Rtexp}\leq M\cdot \frac{T^{2/3}}{K^{2/3}}.\]
Hence, we only need to control $\expect{\Rtcom}$ by the linearity of expectation. To continue, we define the events
\begin{align*}
   \evefindeps{c} = \text{The algorithm finds an arm with reward at least $\mu^*-c\cdot \eps$}
\end{align*}
for every integer $c\geq 1$. Observe that an event $\evefindeps{c}$ \emph{contain} all events with $\evefindeps{c'}$ for $c'<c$. As such, using $\evefindeps{1:c-1}$ as a short-hand notation of the \emph{collection} of events from $\evefindeps{1}$ to $\evefindeps{c-1}$, we note that $\neg  \evefindeps{c-1}$ means \emph{none} of the event from $\evefindeps{1}$ to $\evefindeps{c-1}$ happens. As such, we can re-write the expected regret as:
%
\begin{align*}
\expect{\Rtcom} &= \expect{\Rtcom \mid \evefindeps{1}} \Pr\paren{\evefindeps{1}} + \expect{\Rtcom \mid \neg {\evefindeps{1}}} \Pr\paren{\neg {\evefindeps{1}}}\\
&= \expect{\Rtcom \mid \evefindeps{1}} \Pr\paren{\evefindeps{1}} + \Pr\paren{\neg {\evefindeps{1}}}\cdot \expect{\Rtcom\mid {\evefindeps{2}, \neg \evefindeps{1}}}\Pr\paren{\evefindeps{2}\mid \neg \evefindeps{1}} \\
& \quad + \Pr\paren{\neg {\evefindeps{1}}}\cdot \expect{\Rtcom\mid \neg {\evefindeps{2}}}\Pr\paren{\neg \evefindeps{2}\mid \neg \evefindeps{1}}\\
&=\cdots\\
&=\expect{\Rtcom \mid \evefindeps{1}} \Pr\paren{\evefindeps{1}} \\
&  \quad + \sum_{c=2}^{\infty} \expect{\Rtcom\mid \evefindeps{c}, \neg \evefindeps{c-1}}\Pr\paren{\evefindeps{c} \mid \neg \evefindeps{c-1}}\cdot \Pr\paren{\neg \evefindeps{c-1}}.
\end{align*}
Since $\eps=(\frac{K}{T})^{1/3}$, conditioning on event $\evefindeps{c}$ happens, the regret induced by the committing part is at most
\begin{align*}
\Rtcom\mid \evefindeps{c}, \neg \evefindeps{c-1} &= c \cdot (\frac{K}{T})^{1/3}\cdot T\\
&\leq c\cdot K^{1/3} T^{2/3}.
\end{align*}
On the other hand, recall that the probability for each $\neg \evefindeps{c}$ is at most $\frac{1}{2^c}\cdot \frac{1}{10}$. As such, the probability for $\Pr\paren{\evefindeps{c} \mid \neg \evefindeps{c-1}}\cdot \Pr\paren{\neg \evefindeps{c-1}}$ can be bounded as 
\begin{align*}
\Pr\paren{\evefindeps{c} \mid \neg \evefindeps{c-1}}\cdot \Pr\paren{\neg \evefindeps{c-1}} & \leq \Pr\paren{\neg \evefindeps{c-1}} \tag{$\Pr\paren{\evefindeps{c} \mid \neg \evefindeps{c-1}}\leq 1$} \\
&\leq (\frac{1}{2})^{c-1}\cdot \frac{1}{10}. 
\end{align*}
As such, the expected regret of the committing phase can be bounded as a convergent summation of terms:
\begin{align*}
\expect{\Rtcom} & \leq \frac{1}{10} \cdot K^{1/3}T^{2/3}\cdot \sum_{c=1}^{\infty} \frac{c}{2^{c-1}}\\
& =\frac{2}{5} \cdot K^{1/3}T^{2/3}. \tag{$\sum_{c=1}^{\infty} \frac{c}{2^{c-1}}= 4$}
\end{align*}
Therefore, we have the expected regret to be 
\begin{align*}
\expect{R_{T}} &= \expect{\Rtexp+\Rtcom} \tag{linearity of expectation}\\
&\leq M\cdot \frac{T^{2/3}}{K^{2/3}} + \frac{2}{5}\cdot K^{1/3}T^{2/3} \\
& = O\paren{M\cdot \frac{T^{2/3}}{K^{2/3}}+K^{1/3}T^{2/3}},
\end{align*}
as desired.
\end{proof}

\Cref{lem:smooth-failure-regret} provides a neat approach to bound the expected regret by bounding the number of arm pulls and `smooth failure probability' for $\eps$-best arm algorithms. As we will see shortly, algorithms based on applications of the Chernoff bound generally satisfy the smooth failure probability guarantees. Note that, however, streaming algorithms based on amortized variance analysis (e.g. the single-arm algorithm in \cite{assadi2020exploration}) do \emph{not} generally satisfy this property. As such, we are able to get an $\logstar(K)$-arm memory for regret minimization but not the single-arm memory. It remains an interesting problem to see whether the $\logstar(K)$ bound is an artifact or a necessity.

\subsection{Warm-up: A single-arm memory algorithm with $O(K^{1/3}T^{2/3}\log(K))$ expected regret}
\label{subsec:expect-ub-easy}

To begin with, we first give an algorithm with $O(K^{1/3}T^{2/3}\log(K))$ regret by analyzing the naive uniform elimination algorithm (folklore, see also~\cite{EvenDarMM02}) for $\eps$-best arm. The algorithm is given as follows.

\begin{tbox}
\textbf{\underline{Naive Uniform Elimination}} -- input parameters $\eps\in (0,1)$, $\delta \in (0,1)$
\begin{enumerate}
\item Maintain space of a single extra arm and a best mean reward $\muhat^*$ with initial value $0$.
\item For each arriving $\arm_{i}$ pull $\frac{16}{\eps^2}\log(\frac{K}{\delta})$ times, record the empirical reward $\muhat_{i}$.
\item If $\muhat_{i}>\muhat^*$, discard the stored arm and let the $\arm_{i}$ be the stored arm; update $\muhat^* = \muhat_{i}$.
\item Otherwise, discard $\muhat_{i}$ and keep the stored arm unchanged.
\item Return the stored arm by the end of the stream.
\end{enumerate}
\end{tbox}

It is straightforward to see that the naive uniform elimination algorithm only requires a memory of a single-arm. Furthermore, the total number of arm pulls of the algorithm is clearly $\frac{16K}{\eps^2}\log(\frac{K}{\delta})$. Note that the algorithm description is slightly different from the vanilla Uniform Elimination algorithm as described in \cite{EvenDarMM02} -- the importance of the subtle difference will be clear in the analysis, which we show as the follows.
\begin{lemma}
\label{lem:smooth-guarantee-easy}
For fixed parameters $\delta\in (0,1)$, $\eps\in (0,1)$, and integer $c \geq 1$, the Naive Uniform Elimination algorithm returns an $\overline{\arm}$ with reward 
\[\mu_{\overline{\arm}} \geq \mu^{*} - c \cdot \eps\] 
with probability at least $1-(\frac{1}{2})^{c^2} \cdot \delta$. 
\end{lemma}
\begin{proof}
The lemma is obtained by straightforward applications of the Chernoff bound and \Cref{lem:arm-comp}. Concretely, the arm-pulling line in the Native Uniform Elimination algorithm is equivalent to setting $S=4\log(\frac{K}{\delta})$ for each arm comparison in \Cref{lem:arm-comp}. As such, for a fixed integer $c\geq 1$, when $\armstar$ arrives, it has empirical reward
\begin{align*}
\Pr\paren{\muhat_{\armstar}< \mustar-c\cdot \frac{\eps}{2}} & \leq (\frac{1}{2})^{c^2-1}\cdot \exp(-4\log(\frac{K}{\delta}))\\
& \leq (\frac{1}{2})^{c^2+1}\cdot \frac{\delta}{K}. 
\end{align*}
As such, with probability at least $1-(\frac{1}{2})^{c^2+1}\cdot \frac{\delta}{K}$, the estimation of $\muhat^*$ eventually becomes at least $\mustar-c\cdot \frac{\eps}{2}$. On the other hand, if an $\arm_{i}$ has a reward less than $\mu^*-c\cdot \eps$ we have 
\begin{align*}
\Pr\paren{\muhat_{i}>\mu_{i}+c\cdot \frac{\eps}{2}} & \leq (\frac{1}{2})^{c^2-1}\cdot \exp(-4\log(\frac{K}{\delta}))\\
& \leq (\frac{1}{2})^{c^2+1}\cdot \frac{\delta}{K}. 
\end{align*}
And a union bound over at most $K$ arms gives us that no arm with a mean reward less than $\mu^*-c\cdot \eps$ can be stored in the end with probability at least $(\frac{1}{2})^{c^2+1}\cdot \delta$. Finally, we take a union bound over the failure probability of the aforementioned events, and conclude that with probability at least $1-(\frac{1}{2})^{c^2}\cdot \delta$, the final returned arm is with a mean reward at least $\mu^*-c\cdot \eps$.
\end{proof}

With \Cref{lem:smooth-guarantee-easy} establishing the `smooth failure probability', we can now apply \Cref{lem:smooth-failure-regret} to obtain the regret guarantee for streaming algorithms with uniform elimination. 
\begin{proposition}
\label{prop:ub-expect-easy}
There exists a single-pass streaming algorithm that given a multi-armed bandit instance arriving in a stream with fixed parameters $T$, $K$ such that $T>K$, carries out arm pulls with expected regret $\expect{R_{T}}\leq O(K^{1/3}T^{2/3}\log(K))$ and uses a memory of a single extra arm.
\end{proposition}
\begin{proof}
By \Cref{lem:smooth-guarantee-easy}, we know that for any given parameter $\eps$, there is 
\begin{align*}
\Pr\paren{\mu_{\overline{\arm}} < \mu^{*} - c \cdot \eps} \leq (\frac{1}{2})^{c^2} \cdot \frac{1}{10}\leq (\frac{1}{2})^{c} \cdot \frac{1}{10}.
\end{align*}
by setting $\delta=\frac{1}{10}$. As such, we can match the parameters in \Cref{lem:smooth-failure-regret} by $S=1$ and $M=16K\log(10K)$. This gives us the desired bound of
\begin{align*}
\expect{R_{T}} \leq O\paren{M\cdot \frac{T^{2/3}}{K^{2/3}}+K^{1/3}T^{2/3}} = O(K^{1/3}T^{2/3}\log(K)).
\end{align*}
\end{proof}

\subsection{A $\logstar(K)$-arm memory algorithm with $O(K^{1/3}T^{2/3})$ expected regret}
\label{subsec:expect-ub-full}

We now proceed to our streaming algorithm with the \emph{optimal} expected regret for any streaming algorithm with $o(K)$ memory. Our optimal algorithm follows the same `exploration-and-commit' paradigm, albeit using a non-trivial streaming $\eps$-best arm algorithm recently developed by \cite{assadi2020exploration,MaitiPK21}. Note that although the design of the algorithm is straightforward, we cannot use the algorithm as a blackbox since we need to control the case when the algorithm `fails' -- and it requires a non-trivial proof.

We first give the streaming $\eps$-best arm algorithm with $\logstar(K)$ memory as follows.

\begin{tbox}
\underline{\textbf{Parameter Set 1}}: 
\begin{align*}
& \set{\eps}_{\ell\geq 1}: \eps_{\ell}=\frac{\eps}{10\cdot2^{\ell-1}} \tag{$\eps$ parameter at each level}\\
&\set{r_\ell}_{\ell\geq 1}: r_{1}:=4,\quad r_{\ell+1} = 2^{r_{\ell}}; \qquad \set{\beta_{\ell}}_{\ell\geq 1}: \beta_{\ell}=\frac{1}{\eps_{\ell}^{2}}; \tag{intermediate variables to define $s_\ell$ and $c_\ell$}\\
&\set{s_{\ell}}_{\ell\geq 1}: s_{\ell} = 8\beta_{\ell}(\ln(\frac{1}{\delta})+3r_{\ell}) \tag{number of samples per arm at each level}\\
& \set{c_{\ell}}_{\ell\geq 1} c_{1} = 2^{r_{1}},\quad  c_{\ell}=\frac{2^{r_{\ell}}}{2^{\ell-1}} (\ell \geq 2) \tag{the bound for the number of arms to `defeat' at each level}
\end{align*}
\end{tbox}

\begin{tbox}
\textbf{\underline{Aggressive Selective Promotion}} -- an $\eps$-best arm algorithm using $\log^*(K)$-arm memory  
\smallskip

Counters: $C_{1}, C_{2}, ..., C_{t} \qquad t=\left \lceil{\log^{*}(K)}\right \rceil+1$;\\
Reward records: $\mu_{1}^{*}$, $\mu_{2}^{*}$, ..., $\mu_{t}^{*}$, initialize with $0$; \\
Stored arms: $\arm^{*}_{1}, \arm^{*}_{2}, ..., \arm^{*}_{t}$ the most reward arm of $\ell$-th level.

\begin{itemize}
\item For each arriving $\arm_i$ in the stream do:
\begin{enumerate}[label=($\arabic*$)]
    \item Read $\arm_{i}$ to memory.
    \item Starting from level $\ell=1$: 
    \begin{enumerate}
        \item\label{linesample} Sample $\arm_{i}$ for $s_{\ell}$ times and get $\muhat_{\arm_{i}}$. 
        \begin{enumerate}
        \item If $\muhat_{\arm_{i}}<\mu_{\ell}^{*}$, drop $\arm_{i}$;
        \item Otherwise, replace $\arm^{*}_{\ell}$ with $\arm_{i}$ and set $\mu_{\ell}^{*}=\muhat_{\arm_{i}}$.
        \end{enumerate}
        \item Increase $C_{\ell}$ by 1.
        \item If $C_{\ell}=c_{\ell}$, do 
        \begin{enumerate}
        \item Reset the counter to $C_{\ell}=0$.
        \item Send $\arm^{*}_{\ell}$ to the next level by calling Line~\ref{linesample} with $(\ell = \ell+1)$.
        \end{enumerate}
    \end{enumerate}
\item At the end of the stream
    \begin{enumerate}
        \item For all $i\in [t]$, sample $\arm_{i}^{*}$ for $32\cdot \frac{\logstar(K)}{\eps^2}$ times and get $\muhat^{*}_{i}$.
        \item Return the arm with the highest $\muhat^{*}_{i}$.
    \end{enumerate}
\end{enumerate}
\end{itemize}
\end{tbox}

Unlike the Naive Uniform Elimination algorithm, it is not immediately clear how many arm pulls are used in the Aggressive Selective Promotion algorithm. We can nevertheless use the upper bound on arm pulls in \cite{assadi2020exploration} as a blackbox on this front.

\begin{lemma}[\cite{assadi2020exploration}]
\label{lem:arm-pulls-logstar}
The number of arm pulls used by the Aggressive Selective Promotion algorithm is $O(\frac{K}{\eps^2}\log(\frac{1}{\delta}))$.
\end{lemma}

Note that \Cref{lem:arm-pulls-logstar} holds deterministically without any randomness -- this is simply because of the number of arms reaching higher levels decreases in a towering number speed. On the other hand, similar to our scenario in \Cref{subsec:expect-ub-easy}, it is not immediately clear which arm the Aggressive Selective Promotion algorithm will return if it fails. To this end, we again prove a `smooth version' of success probability for the Aggressive Selective Promotion algorithm. 

\begin{lemma}
\label{lem:smooth-guarantee-logstar}
For fixed parameters $\delta\in (0,1)$, $\eps\in (0,1)$, and integer $c \geq 1$, the Aggressive Selective Promotion algorithm returns an $\arm^{*}_{t}$ with reward 
\[\mu_{\arm^{*}_{t}} \geq \mu^{*} - c \cdot \eps\] 
with probability at least $1-(\frac{1}{2})^{c^2} \cdot \delta$. 
\end{lemma}
\begin{proof}
Fix a level $\ell$, we define the surviving arms of level $\ell$ as the set of arms that can ever reach $\ell$, and let the corresponding mean reward be $\mu_{\ell}$ (pending the randomness of the arms). Our strategy is to argue that with probability at least $\paren{1-(\frac{1}{2})^{c^{2}+2\ell} \cdot \delta}$, the best arm among the surviving arms of level $\ell$ can only be replaced by an arm with mean reward at least $\mu_{\ell} - c\cdot \eps_{\ell}$. Since $\armstar$ is trivially the best arm among the surviving arms of level $1$, this allows us to guarantee the cumulative gap as a summation of $c\dot \eps_{\ell}$ across levels -- a series that converges $c\cdot \eps$.

We now formalize the above strategy. We first show at any level $\ell$, the value of the `benchmark' $\mustar_{\ell}$ does not go below $\mu_{\ell}-\frac{c}{2}\cdot \eps_{\ell}$ with probability at least $\paren{1-(\frac{1}{2})^{c^{2}+3r_{\ell}} \cdot \delta}$. To see this, note that by an application of \Cref{lem:arm-comp}, for any arm with mean reward $\mu$, there is 
\begin{align*}
	\Pr\paren{\muhat \leq \mu - c \cdot \eps_{\ell}/2} &  \leq \exp\paren{- 2c^2 \cdot (\log(\frac{1}{\delta})+3r_{\ell})} \\
	\tag{arm is pulled $s_{\ell} = 8\beta_{\ell}(\ln(\frac{1}{\delta})+3r_{\ell})$ times}\\
	&\leq (\frac{1}{2})^{c^{2}+3r_{\ell}} \cdot \delta.
\end{align*}
As such, let $\mu_{\ell}$ be the mean reward of the best surviving arm of level $\ell$, the empirical reward for $\mu_{\ell}$ is at least $\mu_{\ell}-\frac{c}{2}\cdot \eps_{\ell}$. Suppose the value of $\mustar_{\ell}$ (the benchmark reward) is less than $\mu_{\ell}-\frac{c}{2}\cdot \eps_{\ell}$; then, when $\mu_{\ell}$ joins level $\ell$, the benchmark is updated to the value with probability at least $1-(\frac{1}{2})^{c^{2}+3r_{\ell}} \cdot \delta$.

We then show that at any level $\ell$, any arm with reward less than $\mu_{\ell}-\eps_{\ell}$ can have a empirical reward of at most $\mu_{\ell}-\frac{\eps_{\ell}}{2}$ with probability $1-(\frac{1}{2})^{c^{2}+2r_{\ell}} \cdot \delta$, again by a straightforward application of \Cref{lem:arm-comp}. For an arm with reward $\mu$, there is
\begin{align*}
	\Pr\paren{\muhat \geq \mu + c \cdot \eps_{\ell}/2} &  \leq \exp\paren{- 2c^2 \cdot (\log(\frac{1}{\delta})+3r_{\ell})} \\
	\tag{arm is pulled $s_{\ell} = 8\beta_{\ell}(\ln(\frac{1}{\delta})+3r_{\ell})$ times}\\
	&\leq (\frac{1}{2})^{c^{2}+3r_{\ell}} \cdot \delta.
\end{align*}
As such, we can apply a union bound over the bad events, and obtain that
\begin{align*}
\Pr\paren{\muhat \geq \mu + c \cdot \eps_{\ell}/2 \text{ for any arm on level $\ell$}} &\leq c_{\ell}\cdot (\frac{1}{2})^{c^{2}+3r_{\ell}} \cdot \delta \leq (\frac{1}{2})^{c^{2}+2r_{\ell}} \cdot \delta.
\end{align*}
For any integer $c$, we now have the following statement: by a union bound, with probability at least \[1-\paren{(\frac{1}{2})^{c^{2}+2r_{\ell}} + (\frac{1}{2})^{c^{2}+3r_{\ell}})}\cdot \delta \geq 1-(\frac{1}{2})^{c^{2}+2\ell} \cdot \delta,\]
the benchmark reward on level $\ell$ is at least $\mu_{\ell}-c \cdot \frac{\eps_{\ell}}{2}$, and an arm with such an empirical reward has to have a mean reward of at least $\mu_{\ell}-c \cdot \eps_{\ell}$. Therefore, we conclude that at a fixed level $\ell$ and for any integer $c$, the best $\armstar_{\ell}$ has to have a mean reward at least $\mu_{\ell}-c \cdot \eps_{\ell}$ with probability at least $1-(\frac{1}{2})^{c^{2}+2\ell} \cdot \delta$. We define this high-probability event at level $\ell$ as $\eveepsell$.

Finally, we handle the accumulation of error and failure probability across levels. Note that the failure probability across different levels can be bounded by
\begin{align*}
\Pr\paren{\neg \eveepsell \text{ at any level $\ell$}} &\leq \sum_{\ell=1}^{t} (\frac{1}{2})^{c^{2}+2\ell}\cdot \delta\\
&\leq (\frac{1}{2})^{c^{2}}\delta \sum_{\ell=1}^{\infty} (\frac{1}{2})^{2\ell}\\
&\leq (\frac{1}{2})^{c^{2}} \cdot \delta.
\end{align*}

Conditioning on the high probability event over all levels of $\ell$, the cumulative gap between the best surviving arm on level $1$ (which is $\armstar$) and on level $t$ is at most
\begin{align*}
\sum_{\ell=1}^{t} c \cdot \eps_{\ell} & = c \cdot \sum_{\ell=1}^{\infty} \eps_{\ell}\\
& \leq c \cdot \frac{\eps}{30} \sum_{\ell=1}^{\infty} \frac{1}{2^{\ell-1}}\\
&\leq  c \cdot \eps,
\end{align*}
as desired by the lemma statement.
\end{proof}

We can now arrive at our main $\logstar(K)$-memory regret minimization algorithm by combining \Cref{lem:arm-pulls-logstar,lem:smooth-guarantee-logstar,lem:smooth-failure-regret}. 
\begin{theorem}
\label{prop:ub-expect-logstar}
There exists a single-pass streaming algorithm that given a multi-armed bandit instance arriving in a stream with fixed parameters $T$, $K$ such that $T>K$, carries out arm pulls with expected regret $\expect{R_{T}}\leq O(K^{1/3}T^{2/3})$ and uses a memory of $\left \lceil{\log^{*}(K)}\right \rceil+1$ arms.
\end{theorem}
\begin{proof}
By \Cref{lem:smooth-guarantee-logstar}, for any given parameter $\eps$, there is 
\begin{align*}
\Pr\paren{\mu_{\armstar_{t}} < \mu^{*} - c \cdot \eps} \leq (\frac{1}{2})^{c^2} \cdot \frac{1}{10}\leq (\frac{1}{2})^{c} \cdot \frac{1}{10}.
\end{align*}
by setting $\delta=\frac{1}{10}$. As such, we can match the parameters in \Cref{lem:smooth-failure-regret} by $S=\left \lceil{\log^{*}(K)}\right \rceil+1$ and $M=O(K)$ as in \Cref{lem:arm-pulls-logstar}. This gives us the desired bound of
\begin{align*}
\expect{R_{T}} \leq O\paren{M\cdot \frac{T^{2/3}}{K^{2/3}}+K^{1/3}T^{2/3}} = O(K^{1/3}T^{2/3}),
\end{align*}
which is asymptotically optimal for any streaming algorithm with $o(K)$-arm memory.
\end{proof}

\begin{remark}
\label{rmk:add-algs}
In \cite{assadi2020exploration}, there are additional algorithms with $\log(K)$- and $\log\log(K)$-arm memory that find $\eps$-best arms with $O(\frac{K}{\eps^2})$ arm pulls. Since they follow the same paradigm to apply concentration bounds as in \emph{Aggressive Selective Promotion}, it can be shown that they can also be converted to regret minimization algorithms with the \emph{optimal expected regret}. We provide their algorithmic description in \Cref{subsec:add-algs} without proofs since they are very similar to \Cref{lem:smooth-guarantee-logstar}. We remark that although the memory bounds are worse, for practical implementation, their regret could be smaller than the Aggressive Selective Promotion, and the memory difference is not significant up to $10^{10}$ arms. We will see more on this in \Cref{sec:simulation}.
\end{remark}

\paragraph{A discussion about the single-arm algorithm.} One may naturally wonder whether we can achieve a single-arm memory by the smooth-failure bounded regret lemma -- after all, we are using known algorithms, and the main innovation lies in the analysis. Alas, it appears that at least the single-arm algorithm in \cite{assadi2020exploration} does not follow the property. At a high level, the single-arm algorithm (and a variate that stored $2$ arms, both known as $\algarm$) in \cite{assadi2020exploration} uses the ideas of $(i).$ a multi-level challenge with a geometrically increasing number of arm pulls, and $(ii)$ a ``budget'' the number of arm pulls that is used for a stored arm. They proved that if the stored arm is sufficiently good, say it is the best arm, then with probability at least $99/100$ (or some arbitrary $1-\delta$ by paying $\log(1/\delta)$), the number of arm pulls we used will never exceed a (varying) budget. As such, we can discard an arm whose arriving ``challengers'' uses a large number of arm pulls if we only want to find the best arm with high constant probability.


However, for the expected regret minimization task, with probability $\sim 1/100$, the best arm can actually be discarded, and the algorithm may return an arbitrary arm. One can think of an adversarial instance that uses a considerable number of arms with suboptimal yet ''high enough’’ rewards that ``almost exhaust’’ the sample bound of the stored best arm; then a very bad arm (say with reward 0.0001) comes but still manages to break the sample budget with a small constant probability. Now, the algorithm may commit to this arm, and the expected regret becomes at least $T>> K$. Therefore, it is not immediately clear whether we can achieve $O(1)$-arm for the expected regret minimization in a single pass, and it is an interesting direction to pursue.

%% file: simulation.tex
\section{Implementation and Simulation Results}
\label{sec:simulation}
In this section, we show the empirical evaluation of our algorithms under simulations on Bernoulli arms. In particular, we implemented and tested the uniform exploration algorithm, the naive uniform elimination algorithm, the algorithms from $\eps$-best arm with $O(\log(K))$, $O(\log\log(K))$ and $O(\logstar(K))$ memory, and the $2$-arm $\algarm$ algorithm as in \cite{assadi2020exploration}. The uniform exploration algorithm is used as the benchmark as it is the known best regret minimization algorithm with provable guarantees in a single pass.

Our simulation results find that the proposed algorithm in this paper outperforms the baseline \emph{by a significant margin}. Under all of our setting (each with $50$ runs), there is at least one $\eps$-best arm-based algorithm that achieves 80\% of the benchmark regret on average and 70\% on median, and the margin can be as significant as 70\% of the benchmark when $T$ is large (i.e. ~30\% regret of the benchmark). Across different settings, the best algorithm (the $O(\log\log(K))$-space algorithm) outperforms the uniform exploration algorithm by around 30\% of the mean regret (i.e. 70\% of the benchmark mean regret) and $>50\%$ of the median regret (i.e. $<50\%$ of the benchmark mean regret), while all the $\eps$-best arm-based algorithm outperforms the uniform exploration and the naive elimination in most cases. Interestingly, the $2$-arm $\algarm$ algorithm offers competitive performances, despite being theoretically sub-optimal in (worst-case instance) expected regret. 

\subsection{Simulation and Experiment Settings}
\label{subsec:exp-settings}
We test the algorithms for arms with Bernoulli reward distributions. If the mean of the reward is $\mu$, to simulate a pull of a Bernoulli arm, it suffices to draw a uniform at random sample from $[0,1]$ and see if it is below $\mu$ \footnote{Due to limited computational power, when the number of arm pulls is large, e.g. $>10^5$, we approximate the arm pull result by directly adding a Gaussian noise to $\mu$.}. We construct the stream of arms as a buffer, and the buffer can feed arms to the algorithm whenever needed. In particular, we test two types of streams:
\begin{enumerate}
\item The \emph{uniform reward} setting: all the rewards of the arms are generated uniformly at random from $(0,1)$.
\item The \emph{standout} setting: there is one arm with mean reward $\mu=0.82$, and all other arms are with mean reward $\mu$ drawn from a truncated Gaussian distribution with mean $0.5$ and upper tail cutting-off at $0.8$.
\end{enumerate}
The stream is then ordered randomly by the buffer before it is fed into the algorithms.

We consider the number of arms with $K=500$, $K=5000$, and $K=50000$. In each case, we further consider different number of arm pulls: $T=1000K$, $T=1000K^2$, and $T=1000 K^3$. In the implementation of different algorithms, we keep the leading constant to be $1$ (i.e. we treat $O(\cdot)$ operation as with multiplicative factor of $1$) except for multiplicative factor in the multi-level increment of samples (which we use $1.2$ instead since it has to be $>1$). We also keep the same $\eps$ across levels (as opposed to using $\frac{\eps}{2^{\ell}}$) since the number of levels is small in our experiments. The simulations are all carried on a personal device with Apple M1 chip and 8GB memory, and each setting contains $50$ runs with \emph{fixed} random seeds from $0$ to $49$ for reproducibility.

\subsection{Simulation Results}
\label{subsec:exp-results}
We report the simulation results for each \emph{number of arms} and \emph{type of stream} settings separately, and merge the other factors into separate tables and plots, respectively. The regrets in the tables are in the \emph{relative} scale, i.e., we treat the regret of the uniform exploration algorithm as the benchmark ($1.0$), and compute the relative regrets of other algorithms. Tables tables \ref{tab:K500-uniform} to \ref{tab:K50000-uniform} summarize the mean and median regrets of the uniform reward setting; and Tables \ref{tab:K500-standout} to \ref{tab:K50000-standout} give the mean and median regrets of the \emph{standout} streaming setting, where there is an arm whose reward is much better than others.

\begin{table}[!h]
\centering
\captionsetup{justification=centering}
\caption{\label{tab:K500-uniform}The comparison of the relative regret for different algorithms under setting $K=500$ uniform stream setting.}
\begin{tabular}{|l|P{2cm}|P{1.9cm}|P{1.5cm}|P{1.5cm}|P{1.5cm}|P{1.5cm}|}
\hline
 & Uniform Exploration & Naive Elimination & $\log(K)$ $\eps$-best & $\log\log(K)$ $\eps$-best & $\logstar(K)$ $\eps$-best & Game-of-Arms\\ \hline
\multicolumn{7}{|l|}{Mean Regret} \\ \hline
$T=1000K$ & 1.0 & 2.5732 & 0.5068 & 0.43331 & 0.6790 & 1.1131 \\ \hline
$T=1000K^2$ & 1.0 & 2.3139 & 0.4242 & 0.3670 & 0.6828 & 0.9793 \\ \hline
$T=1000 K^3$ & 1.0 & 1.9321 & 0.8298 & 0.6504 & 0.6411 & 0.9693 \\ \hline
\multicolumn{7}{|l|}{Median Regret} \\ \hline
$T=1000K$ & 1.0 & 2.5816 & 0.5000 & 0.4359 & 0.6757 & 1.1111 \\ \hline
$T=1000K^2$ & 1.0 & 2.3153 & 0.4023 & 0.3604 & 0.6095 & 0.9768 \\ \hline
$T=1000 K^3$ & 1.0 & 2.106 & 0.3941 & 0.3403 & 0.5225 & 0.8701 \\ \hline
\end{tabular}
\end{table}

\begin{table}
\centering
\captionsetup{justification=centering}
\caption{\label{tab:K5000-uniform}The comparison of the relative regret for different algorithms under setting $K=5000$ uniform stream setting.}
\begin{tabular}{|l|P{2cm}|P{1.9cm}|P{1.5cm}|P{1.5cm}|P{1.5cm}|P{1.5cm}|}
\hline
 & Uniform Exploration & Naive Elimination & $\log(K)$ $\eps$-best & $\log\log(K)$ $\eps$-best & $\logstar(K)$ $\eps$-best & Game-of-Arms\\ \hline
\multicolumn{7}{|l|}{Mean Regret} \\ \hline
$T=1000K$ & 1.0 & 3.3398 & 0.4650 & 0.4291 & 0.6328 & 1.0713 \\ \hline
$T=1000K^2$ & 1.0 & 3.1102 & 0.3840 & 0.4773 & 0.6728 & 0.9241 \\ \hline
$T=1000 K^3$ & 1.0 & 2.1362 & 0.5515 & 0.4880 & 0.5639 & 1.0329 \\ \hline
\multicolumn{7}{|l|}{Median Regret} \\ \hline
$T=1000K$ & 1.0 & 3.3684 & 0.4653 & 0.4293 & 0.6379 & 1.0718 \\ \hline
$T=1000K^2$ & 1.0 & 3.0815 & 0.3838 & 0.4285 & 0.6276 & 0.9236 \\ \hline
$T=1000 K^3$ & 1.0 & 2.3269 & 0.4429 & 0.3438 & 0.4334 & 0.9293 \\ \hline
\end{tabular}
\end{table}


\begin{table}[!h]
\centering
\captionsetup{justification=centering}
\caption{\label{tab:K50000-uniform}The comparison of the relative regret for different algorithms under setting $K=50000$ uniform stream setting.}
\begin{tabular}{|l|P{2cm}|P{1.9cm}|P{1.5cm}|P{1.5cm}|P{1.5cm}|P{1.5cm}|}
\hline
 & Uniform Exploration & Naive Elimination & $\log(K)$ $\eps$-best & $\log\log(K)$ $\eps$-best & $\logstar(K)$ $\eps$-best & Game-of-Arms\\ \hline
\multicolumn{7}{|l|}{Mean Regret} \\ \hline
$T=1000K$ & 1.0 & 3.7355 & 0.4274 & 0.4000 & 0.6012 & 1.0290 \\ \hline
$T=1000K^2$ & 1.0 & 3.0423 & 0.5989 & 0.4374 & 0.5733 & 1.2976 \\ \hline
$T=1000 K^3$ & 1.0 & 2.8652 & 0.5686 & 0.5117 & 0.5982 & 1.0960 \\ \hline
\multicolumn{7}{|l|}{Median Regret} \\ \hline
$T=1000K$ & 1.0 & 3.7555 & 0.4264 & 0.3994 & 0.6036 & 1.0008 \\ \hline
$T=1000K^2$ & 1.0 & 3.1974 & 0.5525 & 0.3776 & 0.5713 & 1.1953 \\ \hline
$T=1000 K^3$ & 1.0 & 3.0008 & 0.4393 & 0.3789 & 0.5142 & 0.9996 \\ \hline
\end{tabular}
\end{table}

\begin{table}[!h]
\centering
\captionsetup{justification=centering}
\caption{\label{tab:K500-standout}The comparison of the relative regret for different algorithms under setting $K=500$ standout stream setting.}
\begin{tabular}{|l|P{2cm}|P{1.9cm}|P{1.5cm}|P{1.5cm}|P{1.5cm}|P{1.5cm}|}
\hline
 & Uniform Exploration & Naive Elimination & $\log(K)$ $\eps$-best & $\log\log(K)$ $\eps$-best & $\logstar(K)$ $\eps$-best & Game-of-Arms\\ \hline
\multicolumn{7}{|l|}{Mean Regret} \\ \hline
$T=1000K$ & 1.0 & 2.4357 & 0.6090 & 0.5606 & 0.7883 & 1.2690 \\ \hline
$T=1000K^2$ & 1.0 & 2.3154 & 1.0294 & 0.5368 & 0.7982 & 0.9398 \\ \hline
$T=1000 K^3$ & 1.0 & 2.1100 & 1.9985 & 0.3553 & 1.7054 & 0.8556 \\ \hline
\multicolumn{7}{|l|}{Median Regret} \\ \hline
$T=1000K$ & 1.0 & 2.6234 & 0.4762 & 0.4493 & 0.7530 & 1.1208 \\ \hline
$T=1000K^2$ & 1.0 & 2.3154 & 0.4195 & 0.3915 & 0.6476 & 0.9425 \\ \hline
$T=1000 K^3$ & 1.0 & 2.1100 & 0.3778 & 0.3545 & 0.5900 & 0.8595 \\ \hline
\end{tabular}
\end{table}

\begin{table}
\centering
\captionsetup{justification=centering}
\caption{\label{tab:K5000-standout}The comparison of the relative regret for different algorithms under setting $K=5000$ standout stream setting.}
\begin{tabular}{|l|P{2cm}|P{1.9cm}|P{1.5cm}|P{1.5cm}|P{1.5cm}|P{1.5cm}|}
\hline
 & Uniform Exploration & Naive Elimination & $\log(K)$ $\eps$-best & $\log\log(K)$ $\eps$-best & $\logstar(K)$ $\eps$-best & Game-of-Arms\\ \hline
\multicolumn{7}{|l|}{Mean Regret} \\ \hline
$T=1000K$ & 1.0 & 2.9354 & 0.7552 & 0.6405 & 0.7705 & 1.2104 \\ \hline
$T=1000K^2$ & 1.0 & 2.9551 & 0.8746 & 0.4686 & 0.6112 & 0.8653 \\ \hline
$T=1000 K^3$ & 1.0 & 2.6700 & 9.1241 & 4.3835 & 1.7245 & 0.7826 \\ \hline
\multicolumn{7}{|l|}{Median Regret} \\ \hline
$T=1000K$ & 1.0 & 3.1183 & 0.8236 & 0.6978 & 0.8849 & 1.2310 \\ \hline
$T=1000K^2$ & 1.0 & 2.9551 & 0.3994 & 0.3728 & 0.6122 & 0.8656 \\ \hline
$T=1000 K^3$ & 1.0 & 2.6700 & 0.3622 & 0.3392 & 0.5508 & 0.7835 \\ \hline
\end{tabular}
\end{table}


\begin{table}
\centering
\captionsetup{justification=centering}
\caption{\label{tab:K50000-standout}The comparison of the relative regret for different algorithms under setting $K=50000$ standout stream setting.}
\begin{tabular}{|l|P{2cm}|P{1.9cm}|P{1.5cm}|P{1.5cm}|P{1.5cm}|P{1.5cm}|}
\hline
 & Uniform Exploration & Naive Elimination & $\log(K)$ $\eps$-best & $\log\log(K)$ $\eps$-best & $\logstar(K)$ $\eps$-best & Game-of-Arms\\ \hline
\multicolumn{7}{|l|}{Mean Regret} \\ \hline
$T=1000K$ & 1.0 & 3.0179 & 0.6323 & 0.5341 & 0.7001 & 1.0888 \\ \hline
$T=1000K^2$ & 1.0 & 3.5402 & 0.8447 & 0.5144 & 0.5750 & 1.5099 \\ \hline
$T=1000 K^3$ & 1.0 & 3.1806 & 20.4462 & 0.3182 & 0.5162 & 0.7344 \\ \hline
\multicolumn{7}{|l|}{Median Regret} \\ \hline
$T=1000K$ & 1.0 & 3.0429 & 0.6213 & 0.5421 & 0.7270 & 1.0658 \\ \hline
$T=1000K^2$ & 1.0 & 3.5402 & 0.3780 & 0.3542 & 0.5752 & 0.8194 \\ \hline
$T=1000 K^3$ & 1.0 & 3.1806 & 0.3398 & 0.3185 & 0.5156 & 0.7337 \\ \hline
\end{tabular}
\end{table}


\FloatBarrier

From the tables, it can be observed that the $\eps$-best arm-based algorithms consistent outperform the benchmark uniform exploration. The naive elimination algorithm, on the other hand, offers generally poor performances -- since we only test the number of trials for as large as $1000 K^3$ due to limited computational resource, the term $(\log(T))^{1/3}$ is still much smaller than $\log(K)$. Testing trials with even larger scale will probably help the naive elimination algorithm to catch up in the performance.

Among the $\eps$-best arm algorithms, it appears that the $\log\log(K)$-space algorithm consistently achieve the best mean and median regrets. The $\log(K)$-space algorithm is somehow unstable and offers much worse mean regret in the $K=5000$ and $K=50000$ standalone stream settings. It nonetheless consistently achieves much better median regrets. We suspect this is due to the success probability not sufficiently high, and the algorithm sometimes fails to capture an $\eps$-best arm and commit all remaining trials to a `wrong' arm\footnote{It is likely that this problem can be fixed by a heuristic search for the constant on the $\log(K)$-space algorithm. However, we do not pursue this direction in this paper.}. This also explains why the performance of the $\log(K)$-space algorithm does \emph{not} become worse in the uniform mean-reward setting. Interestingly, the $\algarm$ algorithm in \cite{assadi2020exploration} offers better performance than the naive elimination algorithm, although theoretically, there is a $\frac{\log(K)}{\eps^3}$ term on the sample complexity, which translates into $(T/K)\cdot \log(K)$ regret -- a worse regret bound when $T$ is very large.

We further provides figures of the regrets with the error bars, showing the fluctuations of regrets in each setting in more details \footnote{In the figures, we use $\log(n)$, $\log\log(n)$, and $\logstar(n)$ (using notation of $n$ instead of $K$) in the type of algorithms to keep consistent with original description of algorithms in the pure exploration context. }. Since there are some huge gaps between the regrets with different algorithms, we use $\log_{10}(\cdot)$ scale for the regret.

\begin{figure}
\centering
\includegraphics[width=1.0\textwidth]{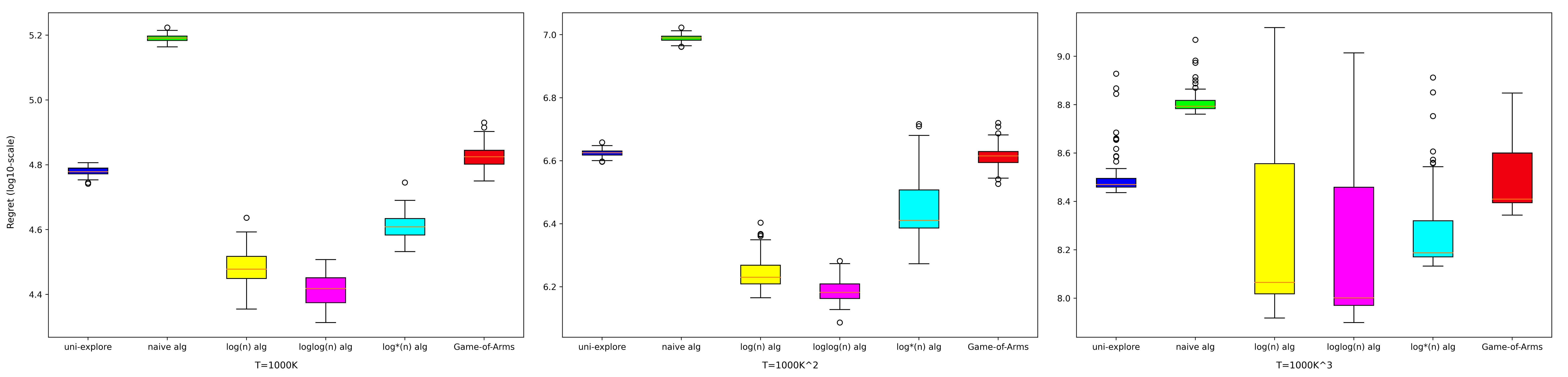}
\caption{\label{fig:K500uniform}The regret error bars for $K=500$ uniform reward setting of the stream.}
\end{figure}

\begin{figure}
\centering
\includegraphics[width=1.0\textwidth]{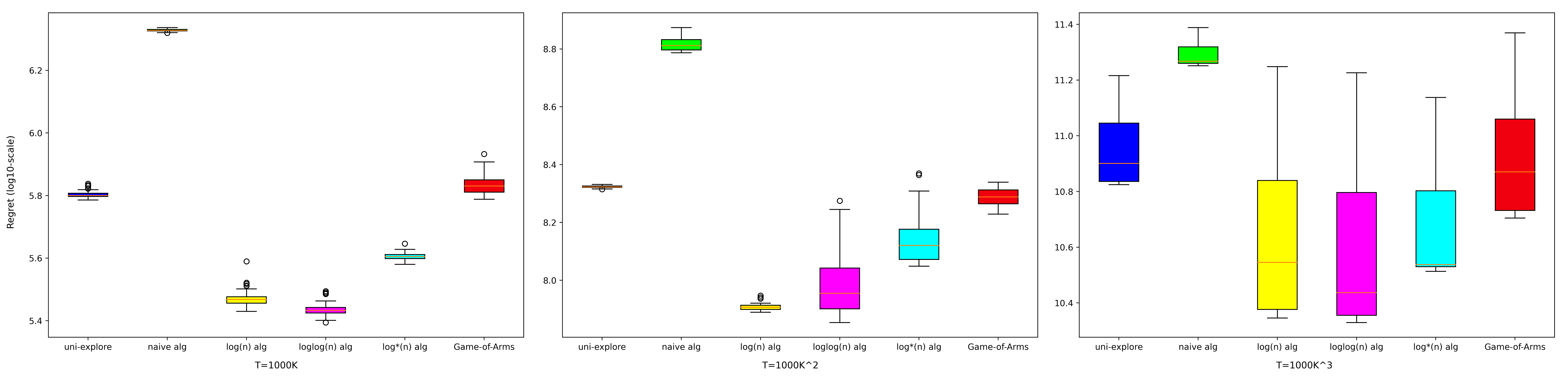}
\caption{\label{fig:K5000uniform}The regret error bars for $K=5000$ uniform reward setting of the stream.}
\end{figure}


\begin{figure}
\centering
\includegraphics[width=1.0\textwidth]{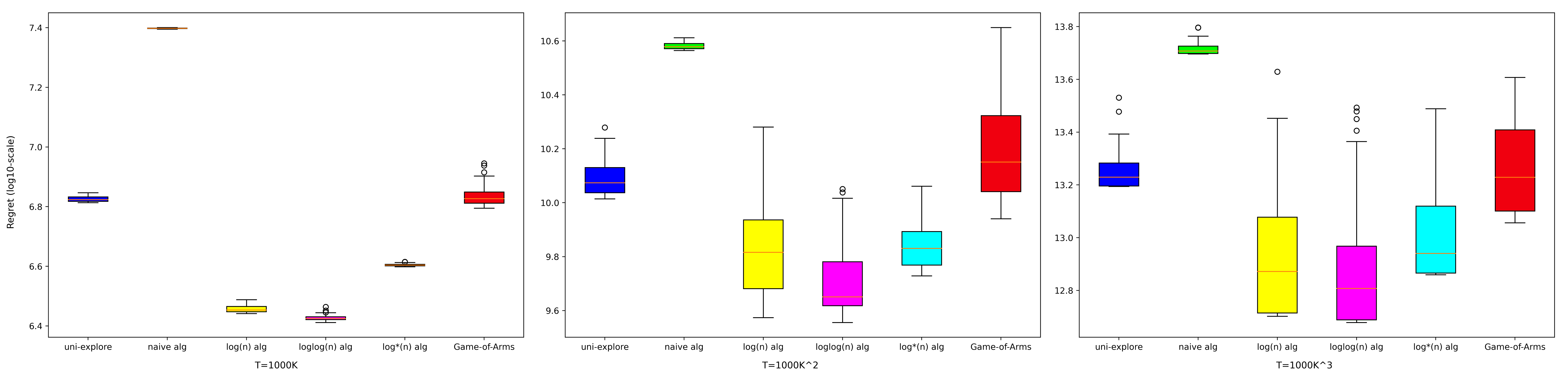}
\caption{\label{fig:K50000uniform}The regret error bars for $K=50000$ uniform reward setting of the stream.}
\end{figure}

\begin{figure}[h]
\centering
\includegraphics[width=1.0\textwidth]{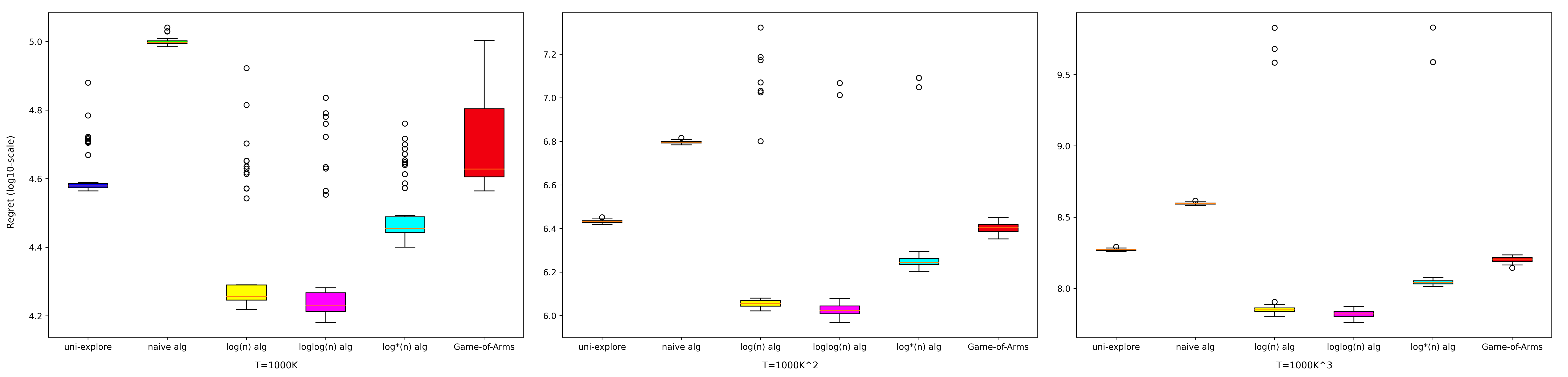}
\caption{\label{fig:K500standout}The regret error bars for $K=500$ standout reward setting of the stream.}
\end{figure}

\begin{figure}
\centering
\includegraphics[width=1.0\textwidth]{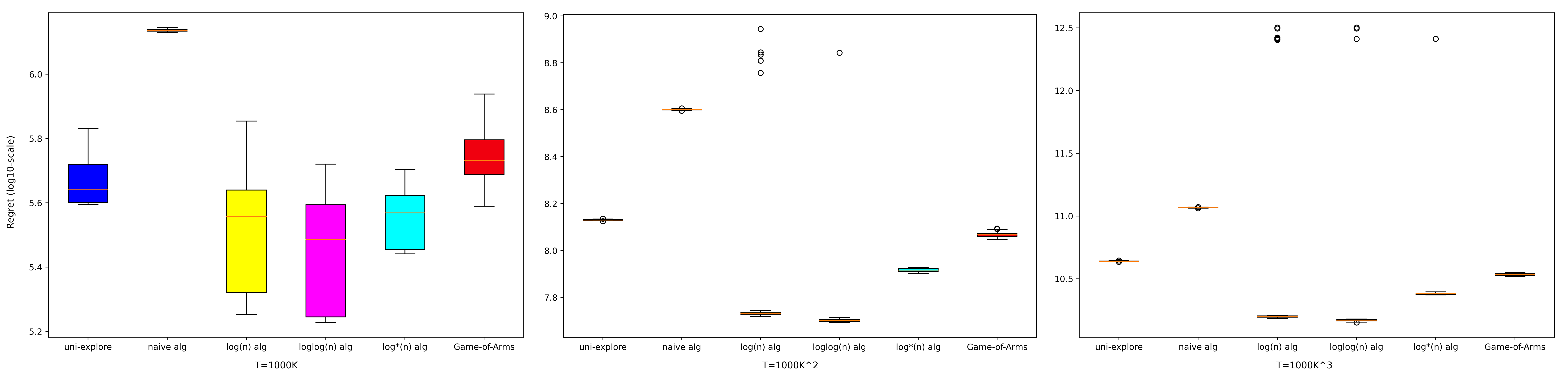}
\caption{\label{fig:K5000standout}The regret error bars for $K=5000$ standout reward setting of the stream.}
\end{figure}


\begin{figure}
\centering
\includegraphics[width=1.0\textwidth]{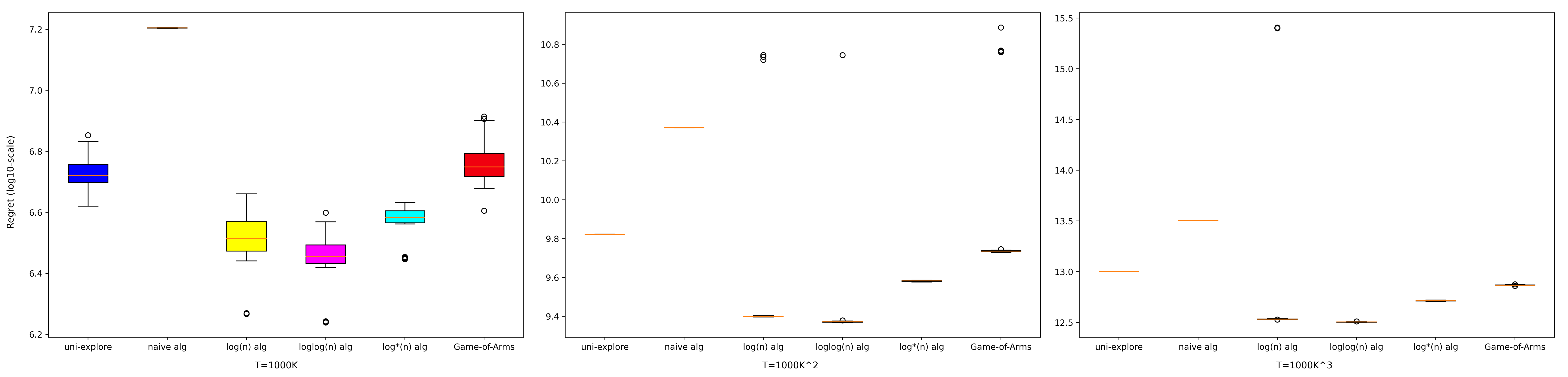}
\caption{\label{fig:K50000standout}The regret error bars for $K=50000$ standout reward setting of the stream.}
\end{figure}

\FloatBarrier

From the figures, it can be observed that the $\log(K)$-space algorithm gives the most unstable performances and the most extreme outliers, while other $\eps$-best algorithms are generally stable. In the uniform reward setting, the first and third quartiles of the rewards do not change drastically w.r.t. $T$, and there are generally less extreme outliers when $T$ is larger. This is because in the uniform reward setting, the differences between the $\eps$-best arms starts to matter, yet the cost of committing to a mediocre arm becomes lower. 
On the other hand, in the standout reward setting, when $T$ is smaller, the first and third quartiles of the reward distributions have larger ranges, but there are generally less extreme outliers. This matches our understanding of the behaviors of the algorithms: when $T$ is smaller, there a good chance that the algorithm terminates before finding an $\eps$-best arm; on the other hand, when $T$ becomes larger, committing to a `wrong' arm is much more expensive.

%% file: conclusion.tex
\section{Conclusion}
\label{sec:conclusion}
In this paper, we studied the tight lower and upper bounds for regret minimization for single-pass streaming multi-armed bandits. In particular, we first improved the regret lower bound for streaming algorithms with $o(K)$ memory from $\max\{\Omega(T^{2/3}), \Omega(K^{1/3}T^{2/3}/m^{7/3})\}$ to $\Omega(K^{1/3}T^{2/3})$, which is tight in both $T$ and $K$. We then proved that the $\Theta(K^{1/3}T^{2/3})$ regret, with high (constant) probability, can be achieved by adopting an $\eps$-best arm algorithm with $O(K/\eps^2)$ arm pulls, setting the parameter $\eps=(K/T)^{1/3}$, and committing to the returned arm. Furthermore, we showed that the simple exploration-and-commit strategy can achieve the \emph{expected} optimal regret of $\Omega(K^{1/3}T^{2/3})$ with a large family of streaming $\eps$-best arm algorithms, and the memory can be as small as $O(\logstar(K))$. Finally, we empirically tested the performances of the $\eps$-best arm-based algorithms on simulations of MABs streams, and we found that the proposed algorithms can significantly outperform the benchmark uniform exploration algorithm.

Our work completes the picture for regret minimization in single-pass streaming MABs with \emph{sublinear} arm memory. On the other hand, it also opens several directions of open problems for future exploration. The first question is whether the memory of arms can be further reduced to $O(1)$ or a single arm, as did in the pure exploration algorithms of \cite{assadi2020exploration} and \cite{JinH0X21}. Note that the single-arm memory algorithm in \cite{assadi2020exploration} may actually return a very bad arm, and it is unclear whether the algorithm in \cite{JinH0X21} has the smooth-failure property. Another open question is the multi-pass setting, where \cite{AgarwalKP22} proved tight bounds for regrets minimization with sublinear arm memory as a function of $T$, but the tight dependent on $K$ is still unclear. Finally, it will be interesting to see the application of our algorithms in real-world scenarios.  

%% file: app-tech-prelim.tex
\section{Technical Preliminaries}
\label{sec:tech-prelim}

We use the following standard variant of Chernoff-Hoeffding bound. 

\begin{proposition}[Chernoff-Hoeffding bound]\label{prop:chernoff}
	Let $X_1,\ldots,X_m$ be $m$ independent random variables with support in $[0,1]$. Define $X := \sum_{i=1}^{m} X_i$. Then, for every $t > 0$, 
	\begin{align*}
		\Pr\paren{\card{X - \expect{X}} > t} \leq 2 \cdot \exp\paren{-\frac{2t^2}{m}}. 
	\end{align*}
\end{proposition}
\noindent
A direct corollary of this bound that we use in our proofs is the following. 
\begin{lemma}\label{lem:arm-comp}
	Let $\arm_1$ and $\arm_2$ be two different arms with rewards $\mu_1$ and $\mu_2$. Suppose we sample each arm $4\cdot \frac{S}{\theta^2}$ times for some $S\geq 2$ to obtain 
	empirical rewards $\muhat_1$ and $\muhat_2$. Then, if ${\mu_1 - \mu_2} \geq c \cdot \theta$ for some integer $c\geq 1$, we have 
	\begin{align*}
		\Pr\paren{\muhat_1 \leq \muhat_2} \leq (\frac{1}{2})^{c^2-1} \cdot \exp\paren{-S}. 
	\end{align*}
\end{lemma}
\begin{proof}
	The proof is a standard application of the Chernoff bound \Cref{prop:chernoff}. For the empirical reward of $\muhat_{2}$ to be greater than $\muhat_{1}$, both of the low-probability following events are neceesary to happen: 
	\begin{align*}
		\Pr\paren{\muhat_1 \leq \mu_1 - c \cdot \theta/2} &\leq \exp\paren{-2\cdot (c \cdot \theta/2)^2 \cdot (4S/\theta^2)} \leq \exp\paren{-c^2\cdot S}; \\
		\Pr\paren{\muhat_2 \geq \mu_2 + c \cdot \theta/2} &\leq \exp\paren{-2\cdot (c \cdot \theta/2)^2 \cdot (4S/\theta^2)} = \exp\paren{-c^2\cdot S}.
	\end{align*}
	For $c=1$, a union bound on the events above gives us the desired bound. For $c\geq 2$, we have 
	\begin{align*}
	\exp\paren{-c^2\cdot S} & \leq \exp\paren{-c^2}\cdot  \exp\paren{-S} \tag{$sc^2\geq s+c^2$ for $S\geq 1$ and $c\geq 2$}\\
	& \leq (\frac{1}{2})^{c^2} \cdot \exp\paren{-S},
	\end{align*}
	and applying a union bound over the two cases gives us the desired statement.
\end{proof}

%% file: app-single-arm-eps-best-alg.tex
\section{The Single-pass Streaming Algorithm for the $\eps$-best Arm}
\label{app:eps-best-arm-alg}

For completeness, we include the algorithm of \cite{JinH0X21} that achieves the property described in \Cref{prop:eps-best-arm-alg}. \cite{assadi2020exploration} achieves a similar guarantee, but their algorithm requires a memory of 2 arms, and the sample complexity as an additive term proportional to $1/\eps^3$. The algorithm of \cite{JinH0X21} can be described as follows.

\begin{tbox}
\underline{\textbf{Parameter Set 2}}: 
\begin{align*}
&\set{s_\ell}_{\ell\geq 0}: r_{0}:=0,\qquad s_{1} := \frac{16}{\eps^2}\cdot \log(\frac{1}{\delta}) \qquad s_{\ell} := (2^{\ell}-2^{\ell-1}) \cdot s_{1} \tag{number of samples used in each level}\\
&\set{\tau_{j}}_{j\geq 1}: \tau_{j}:= \frac{32}{\eps^2}\cdot \log(\frac{j^2}{\delta}) \tag{the ``total budget'' threshold for comparing with the $j$-th arriving arm}\\
& p_{j} = \frac{1}{\log(j)+1} \tag{probability for setting the values of the gap parameter}
\end{align*}
\end{tbox}

\begin{tbox}
\textbf{The Single-pass $\eps$-best Arm Algorithm of \cite{JinH0X21}}
  
\smallskip

\begin{enumerate}
\item Maintain a stored arm $\arm^{o}$ and empirical reward $\muhat^{*}$ of the stored arm.
\item \label{line:start-epoch}After each update of $\arm^{o}$, start an \emph{epoch} as follows:
\begin{enumerate}[label=($\arabic*$)]
\item Let $\arm_j$ be the $j$-th arm after an epoch.
\item Sample $\alpha = \frac{\eps}{4}$ with probability $p_{j}$ and $\alpha=\frac{\eps}{2}$ with probability $1-p_{j}$.
\item Starting from level $\ell=1$: 
    \begin{enumerate}
    \item\label{linesample} Sample $\arm_{j}$ for $s_{\ell}$ times and get $\muhat_{\arm_{j}}$. 
    \item If $\muhat_{\arm_{j}}<\mu_{\ell}^{*}+\alpha$, drop $\arm_{j}$;
    \item Otherwise, if $2^{\ell}\cdot s_{1}>\tau_{j}$, replace $\arm^{o}$ with $\arm_{j}$ and set $\mu_{\ell}^{*}=\muhat_{\arm_{j}}$, and start a new epoch from Line~\ref{line:start-epoch}.
    \item Otherwise, send  $\arm_{j}$ to the next level by calling Line~\ref{linesample} with $(\ell = \ell+1)$.
    \end{enumerate}
\item Output $\arm^{o}$ by the end of the stream.
\end{enumerate}
\end{enumerate}
\end{tbox}

It is easy to observe that the algorithm only uses a memory of a single arm (in addition to the one in the buffer). \cite{JinH0X21} proved that with high probability, the algorithm uses at most $O(\frac{K}{\eps^2}\log(\frac{1}{\delta}))$ arm pulls and returns an $\eps$-best arm.

%% file: app-algorithm.tex
\section{Details for Additional Algorithms}
\label{app:algorithm}
We provide the descriptions for the streaming implementation of the uniform exploration algorithm and the additional algorithms we mentioned in \Cref{rmk:add-algs}. 

\subsection{Uniform Sampling Algorithm under the Streaming Setting}
\label{subsec:uniform-exploration}
We first give the streaming implementation of simple uniform exploration algorithm. The algorithm is to simply pull each arm $N$ times, pick the arm with the highest empirical reward, and commit to the returned arm for the rest of the trials (if any). As such, the streaming adaptation is extremely straightforward:
\begin{tbox}
\textbf{\underline{Streaming Uniform Exploration}} -- parameters $N$: number of arm pulls for each arm
\begin{enumerate}
\item Maintain space of a single extra arm and a best mean reward $\muhat^*$ with initial value $0$.
\item For each arriving $\arm_{i}$ pull $N$ times, record the empirical reward $\muhat_{i}$.
\item If $\muhat_{i}>\muhat^*$, discard the stored arm and let the $\arm_{i}$ be the stored arm; update $\muhat^* = \muhat_{i}$.
\item Otherwise, discard $\muhat_{i}$ and keep the stored arm unchanged.
\item Return the stored arm by the end of the stream.
\end{enumerate}
\end{tbox}

It is easy to see that we only need to main a single arm (in addition to the arriving buffer) during the stream. Furthermore, it is folklore that if we set $N=O((\frac{T}{K})^{2/3}\log^{1/3}(T))$, the expected regret is attained at $O(K^{1/3}T^{2/3}\log^{1/3}(T))$.

\subsection{The $\log(K)$- and $\log\log(K)$-memory streaming algorithms}
\label{subsec:add-algs}
We now introduce the algorithms used with $\log(K)$- and $\log\log(K)$-memory, which are implemented in \Cref{sec:simulation} and at times offer more competitive performances than the $\logstar(K)$-memory algorithm. We opt to include their descriptions since the $\eps$-best algorithms were not described in \cite{assadi2020exploration}. We however omit the proofs for the algorithms to achieve the optimal regret since it is very similar to \Cref{lem:smooth-guarantee-logstar}, and leave it as an exercise for keen readers. For the $O(1)$-memory $\algarm$ algorithm and the single-arm memory algorithm, we refer the reader to the respective work \cite{assadi2020exploration,JinH0X21}. 

The algorithm with $O(\log(K))$ can be described as follows.
\begin{tbox}
	\textbf{\underline{An algorithm with ${O(\log{K})}$-arm space and $O(K^{1/3}T^{2/3})$ expected regret:}}
	
	\smallskip
	
	Input parameter: $K$ number of arms; $T$ number of trials
	
	\smallskip
	
	Parameters:
	\begin{align*}
	& \eps = (\frac{K}{T})^{1/3}\\
	& \eps_{\ell} = \frac{1}{10} \cdot \frac{\eps}{2^{\ell-1}}\\
	& \set{s_\ell}_{\ell \geq 1}: \quad s_{\ell} = \frac{4}{\eps^2_{\ell}} \cdot \paren{\ln{(1/\delta)}+3^{\ell}}.
	\end{align*}
	
	\smallskip
	
	Maintain: 
	\begin{itemize}
	\item Buckets: $B_{1}$, $B_{2}$, ..., $B_{t}$, each of size $4$ for $t :=\ceil{\log_4{(K)}}$. 
	\end{itemize}
	
	\smallskip
	
	Algorithmic procedure:
	\begin{itemize}
	\item For each arriving $\arm_i$ in the stream do:
	\begin{enumerate}[label=($\arabic*$)]
		\item Add $\arm_{i}$ to bucket $B_{1}$. 
		\item\label{line:challenge-send} If any bucket $B_{\ell}$ is full: 
		\begin{enumerate}
			\item We sample each arm in $B_{\ell}$ for $s_{\ell}$ times;
			\item Send the $\arm_{\ell}^{*}$ with the highest empirical reward to $B_{\ell+1}$, an clear the bucket $B_{\ell}$;
		\end{enumerate}
	\end{enumerate}
	\item At the end of the stream, pick the best arm of each bucket with $s_{\ell}$ times, repeat line~\ref{line:challenge-send} regardless of whether the bucket is full.
	\item Pick $\arm_{t}^{*}$ of bucket $B_{t}$ as the selected arm, and commit the rest of the trials to this arm.
	\end{itemize}
\end{tbox}

This algorithm is very similar to the $O(\log(K))$-arm algorithm for best-arm identification in \cite{assadi2020exploration}; in fact, the only technical difference between this algorithm and the original is the usage of exponentially decreasing $\eps$ across levels. We shall note that this is in contrast with the $O(\log\log(K))$-arm and $O(\logstar(K))$-arm algorithms, in which the modification to challenge a \emph{fixed} reward threshold plays an important role. The algorithm with $O(\log\log(K))$-arm space can be shown as follows.

\begin{tbox}
	\textbf{\underline{An algorithm with ${O(\log\log{K})}$-arm space and $O(K^{1/3}T^{2/3})$ expected regret:}}
	
	\smallskip
	
	Input parameter: $K$ number of arms; $T$ number of trials
	
	\smallskip
	
	Parameters:
	\begin{align*}
	& \eps = (\frac{K}{T})^{1/3}\\
	& \eps_{\ell} = \frac{1}{10} \cdot \frac{\eps}{2^{\ell-1}}\\
	& \set{s_\ell}_{\ell \geq 1}: \quad s_{\ell} = \frac{4}{\eps^2_{\ell}} \cdot \paren{\ln{(1/\delta)}+3^{\ell}}; 
	& s_T := \frac{4}{\eps^2} \cdot \paren{\ln{(1/\delta)}+\ln{(K)}}.
	\end{align*}
	
	\smallskip
	
	Maintain: 
	\begin{itemize}
	\item Buckets: $B_{1}$, $B_{2}$, ..., $B_{t-1}$, each of size $4$ for $t :=\ceil{\log_{4}\ln{(K)}}$; $B_{t}$ is of size $1$.
	\item Best-reward on level $t$: $\mu^*_{t}$ initialized to $0$.
	\end{itemize}
	
	\smallskip
	
	Algorithmic procedure:
	\begin{itemize}
	\item For each arriving $\arm_i$ in the stream do:
	\begin{enumerate}[label=($\arabic*$)]
		\item Add $\arm_{i}$ to bucket $B_{1}$. 
		\item For any level $\ell<t$: $B_{\ell}$ is full: 
		\begin{enumerate}
			\item We sample each arm in $B_{\ell}$ for $s_{\ell}$ times;
			\item Send the $\arm_{\ell}^{*}$ with the highest empirical reward to $B_{\ell+1}$, an clear the bucket $B_{\ell}$;
		\end{enumerate}
		\item For level $t$:
		\begin{enumerate}
			\item Sample the most recent arm that reaches level $t$ $s_{t}$ times, reward the empirical reward $\tilde{\mu}$;
			\item If $\tilde{\mu}>\mu^*_{t}$, let the most recent arm be stored, discard the stored arm at level $t$, and update $\mu^*_{t} \leftarrow \tilde{\mu}$;
		\end{enumerate}
	\end{enumerate}
	\item At the end of the stream, pick the best arm of each bucket with $s_{\ell}$ times, and send the best to higher levels even regardless of whether the bucket is full.
	\item Pick the single $\arm_{t}^{*}$ of bucket $B_{t}$ as the selected arm, and commit the rest of the trials to this arm.
	\end{itemize}
\end{tbox}

Note that compared to the $O(\log\log(K))$-arm memory algorithm for best-identification in \cite{assadi2020exploration}, the small yet subtle difference here is that allow more `slack' for the stored arm at level $t$ by not repetitively pulling it and using the fixed reward threshold instead.

%% file: main.bbl
\newcommand{\etalchar}[1]{$^{#1}$}
\begin{thebibliography}{GMMO00}

\bibitem[AAAK17]{AgarwalAAK17}
Arpit Agarwal, Shivani Agarwal, Sepehr Assadi, and Sanjeev Khanna.
\newblock Learning with limited rounds of adaptivity: Coin tossing, multi-armed
  bandits, and ranking from pairwise comparisons.
\newblock In {\em Proceedings of the 30th Conference on Learning Theory, {COLT}
  2017, Amsterdam, The Netherlands, 7-10 July 2017}, pages 39--75, 2017.

\bibitem[ABM10]{AudibertBM10}
Jean{-}Yves Audibert, S{\'{e}}bastien Bubeck, and R{\'{e}}mi Munos.
\newblock Best arm identification in multi-armed bandits.
\newblock In {\em {COLT} 2010 - The 23rd Conference on Learning Theory, Haifa,
  Israel, June 27-29, 2010}, pages 41--53, 2010.

\bibitem[ACE{\etalchar{+}}08]{AgarwalCEMPRRZ08}
Deepak Agarwal, Bee{-}Chung Chen, Pradheep Elango, Nitin Motgi, Seung{-}Taek
  Park, Raghu Ramakrishnan, Scott Roy, and Joe Zachariah.
\newblock Online models for content optimization.
\newblock In {\em Advances in Neural Information Processing Systems 21,
  Proceedings of the Twenty-Second Annual Conference on Neural Information
  Processing Systems, Vancouver, British Columbia, Canada, December 8-11,
  2008}, pages 17--24, 2008.

\bibitem[AKP22]{AgarwalKP22}
Arpit Agarwal, Sanjeev Khanna, and Prathamesh Patil.
\newblock A sharp memory-regret trade-off for multi-pass streaming bandits.
\newblock In Po{-}Ling Loh and Maxim Raginsky, editors, {\em Conference on
  Learning Theory, 2-5 July 2022, London, {UK}}, volume 178 of {\em Proceedings
  of Machine Learning Research}, pages 1423--1462. {PMLR}, 2022.

\bibitem[AMS96]{AlonMS96}
Noga Alon, Yossi Matias, and Mario Szegedy.
\newblock The space complexity of approximating the frequency moments.
\newblock In {\em STOC}, pages 20--29. ACM, 1996.

\bibitem[AW20]{assadi2020exploration}
Sepehr Assadi and Chen Wang.
\newblock Exploration with limited memory: streaming algorithms for coin
  tossing, noisy comparisons, and multi-armed bandits.
\newblock In Konstantin Makarychev, Yury Makarychev, Madhur Tulsiani, Gautam
  Kamath, and Julia Chuzhoy, editors, {\em Proccedings of the 52nd Annual {ACM}
  {SIGACT} Symposium on Theory of Computing, {STOC} 2020, Chicago, IL, USA,
  June 22-26, 2020}, pages 1237--1250. {ACM}, 2020.

\bibitem[AW22]{AWneurips22}
Sepehr Assadi and Chen Wang.
\newblock Single-pass streaming lower bounds for multi-armed bandits
  exploration with instance-sensitive sample complexity.
\newblock In {\em Advances in Neural Information Processing Systems 35: Annual
  Conference on Neural Information Processing Systems 2022, NeurIPS 2022 (to
  appear)}, 2022.

\bibitem[BC12]{bubeck2012regret}
S{\'{e}}bastien Bubeck and Nicol{\`{o}} Cesa{-}Bianchi.
\newblock Regret analysis of stochastic and nonstochastic multi-armed bandit
  problems.
\newblock {\em Found. Trends Mach. Learn.}, 5(1):1--122, 2012.

\bibitem[BF85]{berry1985bandit}
Donald~A Berry and Bert Fristedt.
\newblock Bandit problems: sequential allocation of experiments (monographs on
  statistics and applied probability).
\newblock {\em London: Chapman and Hall}, 5(71-87):7--7, 1985.

\bibitem[CC08]{chow2008adaptive}
Shein-Chung Chow and Mark Chang.
\newblock Adaptive design methods in clinical trials--a review.
\newblock {\em Orphanet journal of rare diseases}, 3(1):1--13, 2008.

\bibitem[CK20]{ChaudhuriK20}
Arghya~Roy Chaudhuri and Shivaram Kalyanakrishnan.
\newblock Regret minimisation in multi-armed bandits using bounded arm memory.
\newblock In {\em The Thirty-Fourth {AAAI} Conference on Artificial
  Intelligence, {AAAI} 2020, The Thirty-Second Innovative Applications of
  Artificial Intelligence Conference, {IAAI} 2020, The Tenth {AAAI} Symposium
  on Educational Advances in Artificial Intelligence, {EAAI} 2020, New York,
  NY, USA, February 7-12, 2020}, pages 10085--10092, 2020.

\bibitem[CL15]{chen2015optimal}
Lijie Chen and Jian Li.
\newblock On the optimal sample complexity for best arm identification.
\newblock {\em CoRR}, abs/1511.03774, 2015.

\bibitem[CLQ17]{ChenLQ17}
Lijie Chen, Jian Li, and Mingda Qiao.
\newblock Towards instance optimal bounds for best arm identification.
\newblock In {\em Proceedings of the 30th Conference on Learning Theory, {COLT}
  2017, Amsterdam, The Netherlands, 7-10 July 2017}, pages 535--592, 2017.

\bibitem[DMR19]{DongMR19}
Shi Dong, Tengyu Ma, and Benjamin~Van Roy.
\newblock On the performance of thompson sampling on logistic bandits.
\newblock In {\em Conference on Learning Theory, {COLT} 2019, 25-28 June 2019,
  Phoenix, AZ, {USA}}, pages 1158--1160, 2019.

\bibitem[DNCP19]{DegenneNCP19}
R{\'{e}}my Degenne, Thomas Nedelec, Cl{\'{e}}ment Calauz{\`{e}}nes, and Vianney
  Perchet.
\newblock Bridging the gap between regret minimization and best arm
  identification, with application to {A/B} tests.
\newblock In Kamalika Chaudhuri and Masashi Sugiyama, editors, {\em The 22nd
  International Conference on Artificial Intelligence and Statistics, {AISTATS}
  2019, 16-18 April 2019, Naha, Okinawa, Japan}, volume~89 of {\em Proceedings
  of Machine Learning Research}, pages 1988--1996. {PMLR}, 2019.

\bibitem[EMM02]{EvenDarMM02}
Eyal Even{-}Dar, Shie Mannor, and Yishay Mansour.
\newblock {PAC} bounds for multi-armed bandit and markov decision processes.
\newblock In {\em Computational Learning Theory, 15th Annual Conference on
  Computational Learning Theory, {COLT} 2002, Sydney, Australia, July 8-10,
  2002, Proceedings}, pages 255--270, 2002.

\bibitem[GMMO00]{GuhaMMO00}
Sudipto Guha, Nina Mishra, Rajeev Motwani, and Liadan O'Callaghan.
\newblock Clustering data streams.
\newblock In {\em 41st Annual Symposium on Foundations of Computer Science,
  {FOCS} 2000, 12-14 November 2000, Redondo Beach, California, {USA}}, pages
  359--366, 2000.

\bibitem[HRR98]{HenzingerRR98}
Monika~Rauch Henzinger, Prabhakar Raghavan, and Sridhar Rajagopalan.
\newblock Computing on data streams.
\newblock In {\em External Memory Algorithms, Proceedings of a {DIMACS}
  Workshop, New Brunswick, New Jersey, USA, May 20-22, 1998}, pages 107--118,
  1998.

\bibitem[JHTX21]{JinH0X21}
Tianyuan Jin, Keke Huang, Jing Tang, and Xiaokui Xiao.
\newblock Optimal streaming algorithms for multi-armed bandits.
\newblock In Marina Meila and Tong Zhang, editors, {\em Proceedings of the 38th
  International Conference on Machine Learning, {ICML} 2021, 18-24 July 2021,
  Virtual Event}, volume 139 of {\em Proceedings of Machine Learning Research},
  pages 5045--5054. {PMLR}, 2021.

\bibitem[JMNB14]{JamiesonMNB14}
Kevin~G. Jamieson, Matthew Malloy, Robert~D. Nowak, and S{\'{e}}bastien Bubeck.
\newblock lil' {UCB} : An optimal exploration algorithm for multi-armed
  bandits.
\newblock In {\em Proceedings of The 27th Conference on Learning Theory, {COLT}
  2014, Barcelona, Spain, June 13-15, 2014}, pages 423--439, 2014.

\bibitem[KCG16]{KaufmannCG16}
Emilie Kaufmann, Olivier Capp{\'{e}}, and Aur{\'{e}}lien Garivier.
\newblock On the complexity of best-arm identification in multi-armed bandit
  models.
\newblock {\em J. Mach. Learn. Res.}, 17:1:1--1:42, 2016.

\bibitem[KHN15]{komiyama2015optimal}
Junpei Komiyama, Junya Honda, and Hiroshi Nakagawa.
\newblock Optimal regret analysis of thompson sampling in stochastic
  multi-armed bandit problem with multiple plays.
\newblock In {\em International Conference on Machine Learning}, pages
  1152--1161. PMLR, 2015.

\bibitem[KKS13]{KarninKS13}
Zohar~Shay Karnin, Tomer Koren, and Oren Somekh.
\newblock Almost optimal exploration in multi-armed bandits.
\newblock In {\em Proceedings of the 30th International Conference on Machine
  Learning, {ICML} 2013, Atlanta, GA, USA, 16-21 June 2013}, pages 1238--1246,
  2013.

\bibitem[KMP13]{KremerMP13}
Ilan Kremer, Yishay Mansour, and Motty Perry.
\newblock Implementing the "wisdom of the crowd".
\newblock In Michael~J. Kearns, R.~Preston McAfee, and {\'{E}}va Tardos,
  editors, {\em Proceedings of the fourteenth {ACM} Conference on Electronic
  Commerce, {EC} 2013, Philadelphia, PA, USA, June 16-20, 2013}, pages
  605--606. {ACM}, 2013.

\bibitem[KZZ20]{Karpov20FOCS}
Nikolai Karpov, Qin Zhang, and Yuan Zhou.
\newblock Collaborative top distribution identifications with limited
  interaction (extended abstract).
\newblock In {\em 61st {IEEE} Annual Symposium on Foundations of Computer
  Science, {FOCS} 2020, Durham, NC, USA, November 16-19, 2020}, pages 160--171.
  {IEEE}, 2020.

\bibitem[LSPY18]{LiauSPY18}
David Liau, Zhao Song, Eric Price, and Ger Yang.
\newblock Stochastic multi-armed bandits in constant space.
\newblock In {\em International Conference on Artificial Intelligence and
  Statistics, {AISTATS} 2018, 9-11 April 2018, Playa Blanca, Lanzarote, Canary
  Islands, Spain}, pages 386--394, 2018.

\bibitem[McG14]{McGregor14}
Andrew McGregor.
\newblock Graph stream algorithms: a survey.
\newblock {\em {SIGMOD} Rec.}, 43(1):9--20, 2014.

\bibitem[MPK21]{MaitiPK21}
Arnab Maiti, Vishakha Patil, and Arindam Khan.
\newblock Multi-armed bandits with bounded arm-memory: Near-optimal guarantees
  for best-arm identification and regret minimization.
\newblock In Marc'Aurelio Ranzato, Alina Beygelzimer, Yann~N. Dauphin, Percy
  Liang, and Jennifer~Wortman Vaughan, editors, {\em Advances in Neural
  Information Processing Systems 34: Annual Conference on Neural Information
  Processing Systems 2021, NeurIPS 2021, December 6-14, 2021, virtual}, pages
  19553--19565, 2021.

\bibitem[MT03]{MannorT03}
Shie Mannor and John~N. Tsitsiklis.
\newblock Lower bounds on the sample complexity of exploration in the
  multi-armed bandit problem.
\newblock In {\em Computational Learning Theory and Kernel Machines, 16th
  Annual Conference on Computational Learning Theory and 7th Kernel Workshop,
  COLT/Kernel 2003, Washington, DC, USA, August 24-27, 2003, Proceedings},
  pages 418--432, 2003.

\bibitem[PZ23]{PengZ23}
Binghui Peng and Fred Zhang.
\newblock Online prediction in sub-linear space.
\newblock In Nikhil Bansal and Viswanath Nagarajan, editors, {\em Proceedings
  of the 2023 {ACM-SIAM} Symposium on Discrete Algorithms, {SODA} 2023,
  Florence, Italy, January 22-25, 2023}, pages 1611--1634. {SIAM}, 2023.

\bibitem[RKJ08]{radlinski2008learning}
Filip Radlinski, Robert Kleinberg, and Thorsten Joachims.
\newblock Learning diverse rankings with multi-armed bandits.
\newblock In {\em Proceedings of the 25th international conference on Machine
  learning}, pages 784--791, 2008.

\bibitem[Rob52]{Robbins1952}
Herbert Robbins.
\newblock Some aspects of the sequential design of experiments.
\newblock {\em Bulletin of the American Mathematical Society}, 58(5):527--535,
  1952.

\bibitem[Sli19]{Slivkins19}
Aleksandrs Slivkins.
\newblock Introduction to multi-armed bandits.
\newblock {\em Found. Trends Mach. Learn.}, 12(1-2):1--286, 2019.

\bibitem[SWXZ22]{SrinivasWXZ22}
Vaidehi Srinivas, David~P. Woodruff, Ziyu Xu, and Samson Zhou.
\newblock Memory bounds for the experts problem.
\newblock In Stefano Leonardi and Anupam Gupta, editors, {\em {STOC} '22: 54th
  Annual {ACM} {SIGACT} Symposium on Theory of Computing, Rome, Italy, June 20
  - 24, 2022}, pages 1158--1171. {ACM}, 2022.

\bibitem[SZ13]{SaureZ13}
Denis Saur{\'{e}} and Assaf Zeevi.
\newblock Optimal dynamic assortment planning with demand learning.
\newblock {\em Manuf. Serv. Oper. Manag.}, 15(3):387--404, 2013.

\bibitem[Tho33]{thompson1933likelihood}
William~R Thompson.
\newblock On the likelihood that one unknown probability exceeds another in
  view of the evidence of two samples.
\newblock {\em Biometrika}, 25(3/4):285--294, 1933.

\bibitem[TZZ19]{TaoZZ19}
Chao Tao, Qin Zhang, and Yuan Zhou.
\newblock Collaborative learning with limited interaction: Tight bounds for
  distributed exploration in multi-armed bandits.
\newblock In {\em In FOCS 2019}, 2019.

\end{thebibliography}
